%% file: colm2026_conference.tex
\documentclass{article} %
\usepackage[preprint]{colm2026_conference}

\usepackage{microtype}
\usepackage[table]{xcolor}
\usepackage{hyperref}
\usepackage{url}
\usepackage{lineno}
\usepackage{booktabs}
\usepackage{algorithm}
\usepackage{algorithmicx}
\usepackage{algpseudocode}
\usepackage{tcolorbox}
\usepackage{subcaption}
\usepackage{amsmath}
\usepackage{amsthm}
\usepackage{multirow}
\usepackage{arydshln}
\usepackage{wrapfig}
\usepackage{amssymb,amsfonts}
\usepackage{graphicx}
\usepackage{cleveref}
\newtheorem{theorem}{Theorem}[section]

\newtheorem{proposition}[theorem]{Proposition}

\theoremstyle{definition}

\theoremstyle{remark}
\newtheorem{remark}[theorem]{Remark}

\definecolor{darkblue}{rgb}{0, 0, 0.5}
\hypersetup{colorlinks=true, citecolor=darkblue, linkcolor=darkblue, urlcolor=darkblue}

\title{F-GRPO: Factorized Group-Relative Policy Optimization for Unified Candidate Generation and Ranking}

\author{
Rohan Surana$^{1}$, Gagan Mundada$^{1}$, Junda Wu$^{1}$, Xintong Li$^{1}$, Yizhu Jiao$^{2}$, Bowen Jin$^{2}$, \\
\textbf{Sizhe Zhou$^{2}$, Tong Yu$^{3}$, Ritwik Sinha$^{3}$, Jiawei Han$^{2}$, Jingbo Shang$^{1}$, Julian McAuley$^{1}$}
\\
$^{1}$UC San Diego \quad
$^{2}$University of Illinois at Urbana-Champaign \quad
$^{3}$Adobe Research \\
\texttt{\{rsurana,gmundada,juw069,xil240,jshang,jmcauley\}@ucsd.edu} \\
\texttt{\{yizhu,bowenj4,sizhez,hanj\}@illinois.edu} \quad
\texttt{\{tyu,ritwik\}@adobe.com} \\
}

\begin{document}

\ifcolmsubmission
\linenumbers
\fi

\maketitle

\input{contents/0_abstract}

\input{contents/1_introduction}

\input{contents/3.1_formulation}
\input{contents/4_methodology}
\input{contents/5_experiments}

\input{contents/2_related_new}

\input{contents/6_conclusion}

\input{colm2026_conference.bbl}
\bibliographystyle{colm2026_conference}

\appendix

\input{contents/7_appendix}
\end{document}

%% file: contents/0_abstract.tex
\begin{abstract}

Traditional retrieval pipelines optimize utility through stages of candidate retrieval and reranking, where ranking operates over a predefined candidate set.
Large Language Models (LLMs) broaden this into a generative process: given a candidate pool, an LLM can generate a subset and order it within a single autoregressive pass.
However, this flexibility introduces a new optimization challenge: the model must search a combinatorial output space while receiving utility feedback only after the full ranked list is generated. 
Because this feedback is defined over the completed sequence, 
it cannot distinguish whether a poor result arises from failing to generate a relevant subset or from failing to rank that subset correctly. 
This credit assignment gap makes end-to-end optimization both unstable and sample-inefficient.
Existing systems often address this challenge by separating candidate generation from ranking. 
However, such decoupling remains misaligned with downstream utility because the ranking stage is fundamentally limited by the candidate set it receives.
To bridge the optimization gap between candidate generation and ranking, 
we propose a unified framework that performs both within a single autoregressive rollout and optimizes them end-to-end via factorized group-relative policy optimization (F-GRPO).
Our framework factorizes the policy into candidate generation and ranking while sharing a single LLM backbone across both stages, 
and jointly trains them with an order-invariant coverage reward and a position-aware utility reward. 
To address the resulting phase-specific credit assignment problem, we use separate group-relative advantages for generation and ranking within a two-phase sequence-level objective.
Across sequential recommendation and multi-hop question answering benchmarks, F-GRPO improves top-ranked performance over the GRPO and decoupled baselines, outperforms supervised alternatives, and remains competitive with strong zero-shot rerankers, with no architectural changes at inference time.

\end{abstract}

%% file: contents/1_introduction.tex
\section{Introduction}

Retrieval and recommendation systems often face a coupled \emph{list-to-rank} problem, in which the system must identify a slate of relevant candidates and order them so that the best appear first~\citep{10.1561/1500000019,liu2009learning,ni2025large,li2025personalized,huang2025survey,xiesurvey,yanlist,10.1145/3474085.3475366}. Traditionally, this is handled by multi-stage pipelines that first retrieve candidates and then rerank them~\citep{nogueira-etal-2020-document,glass-etal-2022-re2g,yue2023llamarectwostagerecommendationusing,ni2026survey,wu2025doc,hu2025interactive,xie2024neighborhood,wu2024coral}. Large Language Models (LLMs) broaden this design space by serving as stronger rerankers and direct listwise rankers over retrieved candidate sets~\citep{sun-etal-2023-chatgpt,10.1007/978-3-031-56060-6_24,surana2025reviews}.

However, existing LLM-based approaches handle this coupling imperfectly. Some treat the LLM as a direct reranker over a fixed retrieved candidate pool~\citep{sun-etal-2023-chatgpt,10.1007/978-3-031-56060-6_24,xia2025sand,wu2025collap,xia2025knowledge}, so the model outputs only a final ranking and provides no handle for separating candidate coverage from ordering quality (Figure~\ref{fig:motivation-blackbox}). Others decompose the problem into separate modules~\citep{yue2023llamarectwostagerecommendationusing,trivedi-etal-2023-interleaving,huang2025towards}, such as a retriever followed by a reranker. While effective, this separation introduces additional models and optimizes proxy objectives for each stage independently~\citep{10.1145/3731120.3744583,10.1145/3477495.3531969,10.1145/3459637.3482328}, rather than optimizing the coupled list-to-rank decision end-to-end.

\begin{wrapfigure}{r}{0.48\textwidth}
    \vspace{-1.0em}
    \centering
    \begin{subfigure}{0.47\textwidth}
        \centering
        \includegraphics[width=\linewidth]{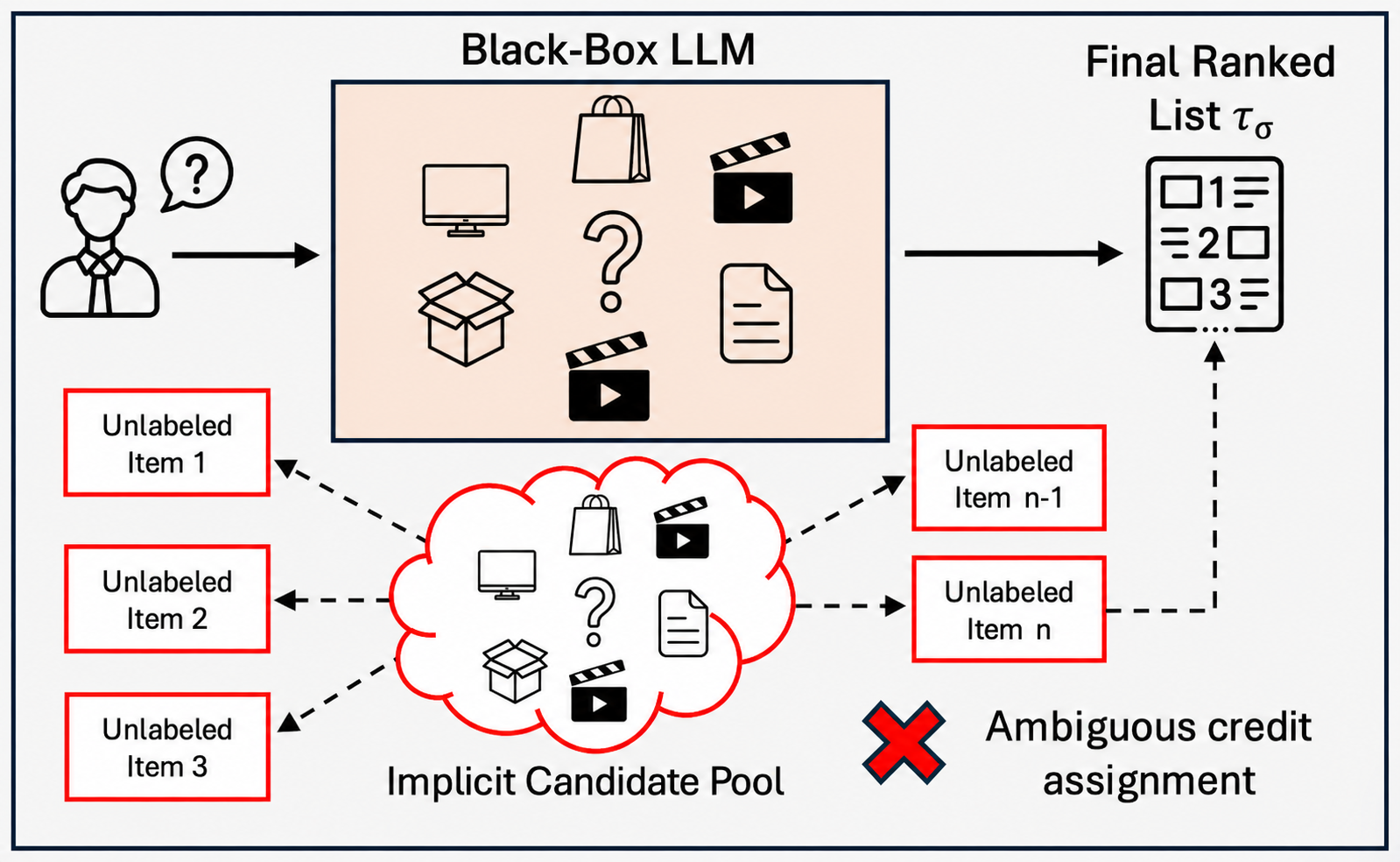}
        \caption{Black-box LLM with implicit candidate pool.}
        \label{fig:motivation-blackbox}
    \end{subfigure}
    \vspace{-0.3em}
    \begin{subfigure}{0.47\textwidth}
        \centering
        \includegraphics[width=\linewidth]{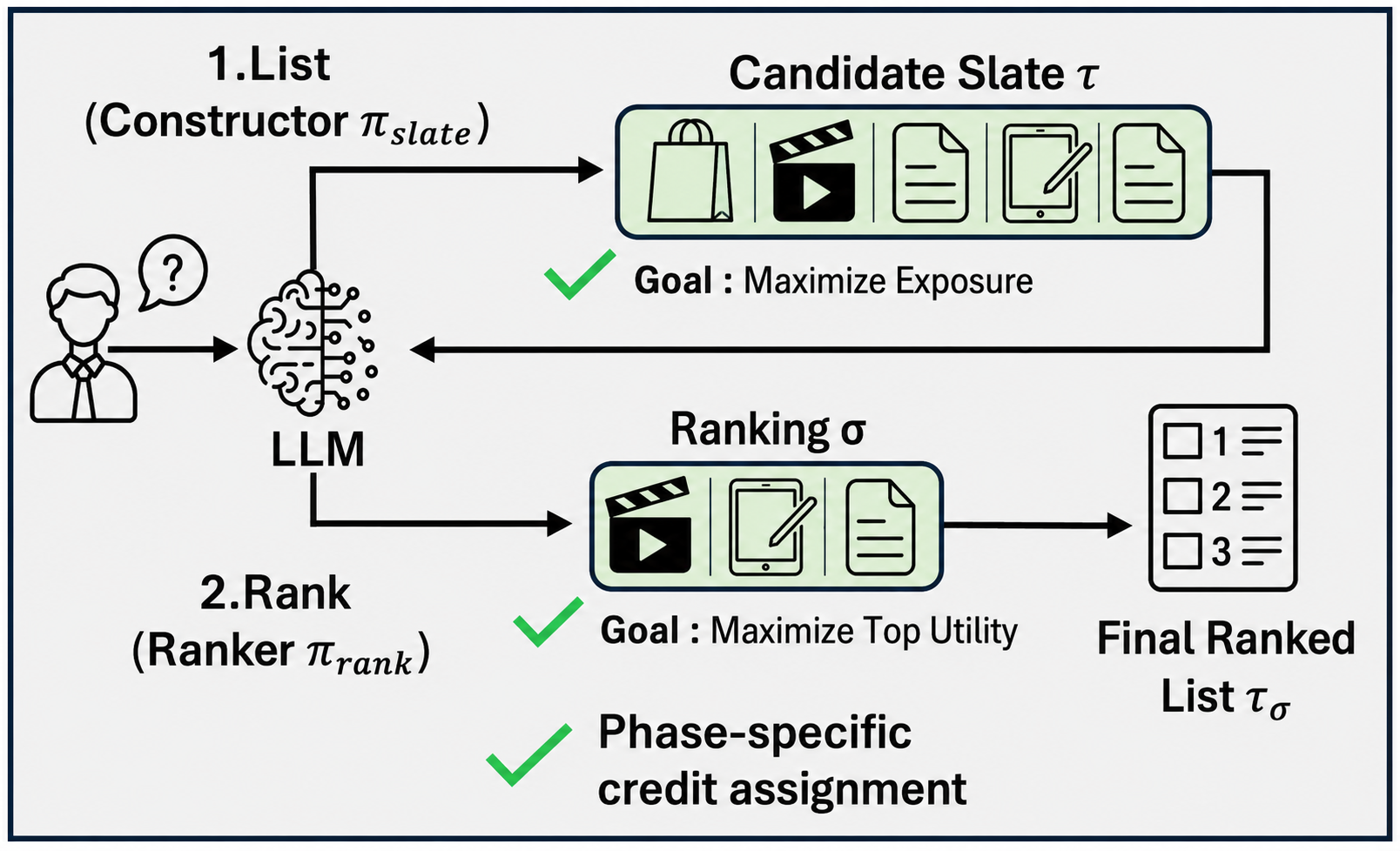}
        \caption{Factorized in-context generation and ranking with phase-specific goals.}
        \label{fig:motivation-factorized}
    \end{subfigure}
\caption{%
  \textbf{(a)}~Black-box ranking conflates candidate selection and 
  ordering, yielding ambiguous credit assignment.
  \textbf{(b)}~Factorized in-context generation and ranking with 
  phase-specific goals within a single autoregressive rollout.%
}
    \label{fig:motivation}
\end{wrapfigure}
We make the list-to-rank decision explicit within a single LLM rollout through \emph{in-context exploration}. The model first constructs a candidate slate and then ranks that slate within the same autoregressive trajectory (Figure~\ref{fig:motivation-factorized}). We term this \emph{in-context exploration} because the slate phase explicitly searches over the candidate space before committing to a ranking. Keeping both phases in the same context allows the ranking phase to condition directly on the generated candidates, enabling end-to-end optimization without the separate modules required by staged pipelines~\citep{yue2023llamarectwostagerecommendationusing,trivedi-etal-2023-interleaving,wu2024personalized}.

A central challenge is that a single sequence-level reward conflates slate construction and ranking, so the same feedback simultaneously rewards coverage and ordering~\citep{wu2025incontext,huang2026evaluation}. We therefore propose \emph{F-GRPO}, which extends GRPO~\citep{shao2024deepseekmathpushinglimitsmathematical,guo2025deepseek,mundada2026ws,surana2026generatefiltercontrolreplay} with \emph{two-phase sequence-level credit assignment} by assigning separate group-relative advantages to slate construction and ranking.

We evaluate F-GRPO on sequential recommendation (MovieLens~\citep{10.1145/2827872}, LastFM~\citep{Bertin-Mahieux2011}) and multi-hop question answering (HotpotQA~\citep{yang-etal-2018-hotpotqa}, MuSiQue~\citep{trivedi-etal-2022-musique}) using Qwen3-4B~\citep{qwen3technicalreport} and Qwen3.5-2B~\citep{qwen3.5} (\S\ref{sec:results_rec}, \S\ref{sec:results_qa}). Factorized credit assignment improves over GRPO, with the clearest gains at higher cutoffs and on settings where coverage is limiting. Analysis of training dynamics confirms the expected phase separation: the slate generator matures before the ranker, and error attribution is balanced across both phases, validating that each requires its own learning signal (\S\ref{sec:analysis}).

\paragraph{Contributions.} Our main contributions are as follows:
\begin{itemize}
      \item We formalize in-context generation and ranking as a factorized policy over a proposal slate and a permutation, defining an objective that jointly optimizes coverage and    
  listwise ordering within a single autoregressive rollout (\S\ref{sec:problem}).

\item We propose F-GRPO, a two-phase sequence-level GRPO method with group-relative advantages for each phase, and evaluate it on sequential recommendation and multi-hop question answering using slate and ranking rewards (\S\ref{sec:method}, \S\ref{sec:constructor_instantiations}).

\item Extensive empirical evaluations across diverse ranking tasks demonstrate that F-GRPO
achieves consistent improvements in Recall@$k$ and NDCG@$k$ over zero-shot, supervised, decoupled, and GRPO baselines (\S\ref{sec:results_rec}, \S\ref{sec:results_qa}).
\end{itemize}

%% file: contents/3.1_formulation.tex
\section{Problem Formulation}
\label{sec:problem}

We study a coupled list-to-rank decision in which, for each context $x \in \mathcal{X}$, a model must first construct a slate of candidates and then order that slate.
Let $\mathcal{E}$ denote the candidate space. Our goal is to optimize both (i) the \emph{coverage and quality} of the candidates surfaced by the model, and (ii) the \emph{ordering} so that higher-utility candidates appear earlier in the list.
The context $x$ absorbs all available information so that each task reduces to one coupled slate-construction-and-ranking decision.

\paragraph{Slate generation and in-context ranking as a factorized policy.}
Let $\tau = (e_1, \ldots, e_n)$ denote an ordered \emph{proposal slate} of $n$ candidates with $e_i \in \mathcal{E}$, and let $\sigma \in \mathcal{S}_n$ (the set of all permutations of $[n] = \{1, \ldots, n\}$) be a permutation specifying the final \emph{in-context ranked} ordering of these candidates.
We model the joint decision $(\tau, \sigma)$ with a factorized policy:
\begin{equation}
\label{eq:factorized}
\pi_\theta(\tau, \sigma \mid x)
= \pi^{\mathrm{slate}}_\theta(\tau \mid x) \cdot \pi^{\mathrm{rank}}_\theta(\sigma \mid x, \tau),
\end{equation}
where $\pi^{\mathrm{slate}}_\theta$ is the slate generator and $\pi^{\mathrm{rank}}_\theta$ is the ranker conditioned on the generated slate.

Both $\pi^{\mathrm{slate}}_\theta$ and $\pi^{\mathrm{rank}}_\theta$ are realized by a single autoregressive model with shared parameters: the model first generates tokens defining the slate $\tau$, then produces an ordering over those candidates within the same decoding trajectory.

\paragraph{Feedback signals.}
Let $U: \mathcal{X} \times \mathcal{E} \to \mathbb{R}_{\geq 0}$ be a task-dependent utility function, and let $\mathcal{G}(x) = \{e \in \mathcal{E} : U(x, e) > 0\}$ denote the set of \emph{gold} (relevant) items for context $x$.
The \emph{slate generator} receives an order-invariant reward over the relevant proposed items, while the \emph{ranker} receives a position-aware reward over the reordered list $\tau_\sigma = (e_{\sigma(1)}, \ldots, e_{\sigma(n)})$: $R_{\mathrm{slate}}(x, \tau) \triangleq \sum_{e \in \mathrm{uniq}(\tau)} U(x, e), \; R_{\mathrm{rank}}(x, \tau, \sigma) \triangleq U_{\mathrm{rank}}(x, \tau_\sigma).$
Here, $\mathrm{uniq}(\tau)$ denotes the set of distinct items in $\tau$.
These signals differ structurally: $R_{\mathrm{slate}}$ is order-invariant and depends only on coverage of relevant items, whereas $R_{\mathrm{rank}}$ is position-sensitive and depends on the ordering induced by $\sigma$. A single scalar reward applied uniformly across both phases would therefore couple two distinct objectives and obscure which phase caused success or failure.
The specific instantiations of $U$ and $U_{\mathrm{rank}}$ are given in Section~\ref{sec:constructor_instantiations}.

\paragraph{Learning objective.}
Let $\mathcal{D}$ denote the distribution over contexts.
We optimize:
\begin{equation}
\label{eq:objective}
\max_{\theta} \;
\mathbb{E}_{x \sim \mathcal{D}} \,
\mathbb{E}_{\tau \sim \pi^{\mathrm{slate}}_\theta(\cdot \mid x)} \,
\mathbb{E}_{\sigma \sim \pi^{\mathrm{rank}}_\theta(\cdot \mid x, \tau)}
\Bigl[ R_{\mathrm{slate}}(x, \tau) + \lambda \, R_{\mathrm{rank}}(x, \tau, \sigma) \Bigr],
\end{equation}
where $\lambda \geq 0$ controls the trade-off between candidate coverage and ranking quality.
Section~\ref{sec:method} shows how to optimize this objective using group-relative policy gradients with factorized credit assignment.

%% file: contents/4_methodology.tex
\section{F-GRPO: Factorized Group-Relative Policy Optimization}
\label{sec:method}

Given the two-phase structure in Section~\ref{sec:problem}, we optimize Eq.~\eqref{eq:objective} with GRPO and then specialize it to factorized credit assignment.

\subsection{Phase-Specific Losses}
\label{sec:unified_opt}

For each context $x$, GRPO samples a group of $G$ rollouts from the current sampling policy $\pi_{\theta_{\mathrm{old}}}$.
We parse each rollout $o^{(i)}$ into two segments: the slate content $c^{(i)}_{\tau}$ and the rank content $c^{(i)}_{\sigma}$, delimited by tag pairs (\texttt{<SLATE>}\ldots\texttt{</SLATE>} and \texttt{<RANK>}\ldots\texttt{</RANK>}). Each raw rollout $o^{(i)}$ is therefore the token-level realization of the decision pair $(\tau^{(i)}, \sigma^{(i)})$, with $c^{(i)}_{\tau}$ decoding to $\tau^{(i)}$ and $c^{(i)}_{\sigma}$ decoding to $\sigma^{(i)}$.
Delimiter tokens are included in the forward pass so that content tokens are conditioned on the correct tagged prefix, but are excluded from the loss via position-based masking (details in Appendix~\ref{app:delimiter_masking}). For readability, the conditioning expressions below suppress these fixed delimiter tokens.
At the full-sequence level, standard GRPO optimizes the clipped objective~\citep{shao2024deepseekmathpushinglimitsmathematical}:
\begin{equation}
\label{eq:grpo_objective}
\mathcal{J}_{\mathrm{GRPO}}(\theta) = \mathbb{E}\!\left[
\frac{1}{G} \sum_{i=1}^{G} \frac{1}{|o^{(i)}|} \sum_{t=1}^{|o^{(i)}|}
\min\!\Bigl( \rho^{(i)}_t \hat{A}^{(i)},\, \mathrm{clip}\bigl(\rho^{(i)}_t, 1{\pm}\epsilon_{\mathrm{clip}}\bigr) \hat{A}^{(i)} \Bigr)
- \beta_{\mathrm{KL}} \, D_{\mathrm{KL}}\!\bigl(\pi_\theta \,\|\, \pi_{\mathrm{ref}}\bigr) \right],
\end{equation}
where $\rho^{(i)}_t = \pi_\theta(o^{(i)}_t \mid x, o^{(i)}_{<t}) / \pi_{\theta_{\mathrm{old}}}(o^{(i)}_t \mid x, o^{(i)}_{<t})$ is the per-token importance ratio, $\hat{A}^{(i)}$ is the rollout-level group-relative advantage, $\epsilon_{\mathrm{clip}}$ is the clipping threshold, $\beta_{\mathrm{KL}}$ is the KL regularization coefficient, and $\pi_{\mathrm{ref}}$ is a frozen reference policy.

\subsection{Factorized Credit Assignment}
\label{sec:factorized_credit}

Each rollout $i$ produces a slate $\tau^{(i)}$ and a ranking permutation $\sigma^{(i)}$, yielding two scalar rewards: $R^{(i)}_{\mathrm{slate}} = R_{\mathrm{slate}}(x, \tau^{(i)})$ and $R^{(i)}_{\mathrm{rank}} = R_{\mathrm{rank}}(x, \tau^{(i)}, \sigma^{(i)})$ as defined in Section~\ref{sec:problem}.
Rather than combining these into a single scalar, we compute separate group-relative advantages using the mean-subtracted Dr.~GRPO variant:
\begin{equation}
\label{eq:separate_advantages}
\hat{A}^{(i)}_{\mathrm{slate}} = R^{(i)}_{\mathrm{slate}} - \bar{R}_{\mathrm{slate}}, \qquad
\hat{A}^{(i)}_{\mathrm{rank}} = R^{(i)}_{\mathrm{rank}} - \bar{R}_{\mathrm{rank}},
\end{equation}
where $\bar{R}_{\mathrm{slate}} = \frac{1}{G} \sum_{j=1}^{G} R^{(j)}_{\mathrm{slate}}$ and $\bar{R}_{\mathrm{rank}} = \frac{1}{G} \sum_{j=1}^{G} R^{(j)}_{\mathrm{rank}}$ are the per-prompt group means over the $G$ rollouts.
This phase-specific advantage assignment is the key departure from standard GRPO: standard GRPO applies one rollout-level advantage to all tokens in a completion, whereas F-GRPO applies different rollout-level advantages to the slate and rank token subsequences within the same autoregressive rollout.
This decoupling ensures that the slate generator receives gradient signal purely from coverage quality, while the ranker receives signal purely from ordering quality. The complete training procedure is summarized in \Cref{alg:training}.

\paragraph{Slate loss.}
The slate loss optimizes the probability of generating slate content conditioned on the prompt $x$:
\begin{equation}
\label{eq:slate_loss}
\mathcal{L}_{\mathrm{slate}} = -\frac{1}{G} \sum_{i=1}^{G} \frac{1}{|c^{(i)}_\tau|} \sum_{t=1}^{|c^{(i)}_\tau|} \min\!\bigl( \rho^{(i)}_{\tau,t}\, \hat{A}^{(i)}_{\mathrm{slate}},\; \mathrm{clip}(\rho^{(i)}_{\tau,t},\, 1{-}\epsilon_{\mathrm{clip}},\, 1{+}\epsilon_{\mathrm{clip}})\, \hat{A}^{(i)}_{\mathrm{slate}} \bigr),
\end{equation}
where $\rho^{(i)}_{\tau,t} = \pi_\theta(c^{(i)}_{\tau,t} \mid x, c^{(i)}_{\tau,<t}) \,/\, \pi_{\theta_{\mathrm{old}}}(c^{(i)}_{\tau,t} \mid x, c^{(i)}_{\tau,<t})$ is the per-token importance ratio and the advantage $\hat{A}^{(i)}_{\mathrm{slate}}$ is uniform across all slate tokens.

\paragraph{Rank loss.}
The rank loss optimizes the probability of generating ranking content conditioned on the prompt \emph{augmented with the generated slate}:
\begin{equation}
\label{eq:rank_loss}
\mathcal{L}_{\mathrm{rank}} = -\frac{1}{G} \sum_{i=1}^{G} \frac{1}{|c^{(i)}_\sigma|} \sum_{t=1}^{|c^{(i)}_\sigma|} \min\!\bigl( \rho^{(i)}_{\sigma,t}\, \hat{A}^{(i)}_{\mathrm{rank}},\; \mathrm{clip}(\rho^{(i)}_{\sigma,t},\, 1{-}\epsilon_{\mathrm{clip}},\, 1{+}\epsilon_{\mathrm{clip}})\, \hat{A}^{(i)}_{\mathrm{rank}} \bigr),
\end{equation}
where $\rho^{(i)}_{\sigma,t}$ is defined analogously with context $(x, \tau^{(i)})$.
Conditioning on the slate ensures the ranker learns to order the specific candidates it was given, preserving the autoregressive dependency between phases.

\paragraph{Combined objective.}
The total loss combines both phases with optional KL regularization:
\begin{equation}
\label{eq:combined_loss}
\mathcal{L}(\theta) = \mathcal{L}_{\mathrm{slate}} + \lambda \, \mathcal{L}_{\mathrm{rank}} + \beta_{\mathrm{KL}} \, D_{\mathrm{KL}}(\pi_\theta \| \pi_{\mathrm{ref}}).
\end{equation}
As shown in Appendix~\ref{app:theory}, this loss admits a first-order decomposition into slate and rank GRPO-style gradients on shared parameters, implemented with a single backward pass through the combined objective.
During training, rollouts that fail to produce the required delimiter tags receive a constant format penalty $p_{\mathrm{fmt}} < 0$ in place of the computed reward; Appendix~\ref{app:three_case} specifies how this penalty is applied to the generated tokens in malformed cases.

\subsection{Gradient Analysis}
\label{sec:gradient_analysis}

For comparison, consider a GRPO baseline that defines a single combined reward $R^{(i)}_{\mathrm{joint}} = R^{(i)}_{\mathrm{slate}} + \lambda \, R^{(i)}_{\mathrm{rank}}$ and computes a joint group-relative advantage
\begin{equation}
\label{eq:joint_advantage}
\hat{A}^{(i)}_{\mathrm{joint}} = R^{(i)}_{\mathrm{joint}} - \bar{R}_{\mathrm{joint}},
\end{equation}
where $\bar{R}_{\mathrm{joint}} = \frac{1}{G} \sum_{j=1}^{G} R^{(j)}_{\mathrm{joint}}$ is the corresponding per-prompt group mean. This joint advantage is applied uniformly to all tokens in rollout~$i$.
At $\theta = \theta_{\mathrm{old}}$, let $\mathcal{T}^{(i)}_\tau$ and $\mathcal{T}^{(i)}_\sigma$ denote the sets of token positions belonging to the slate content $c^{(i)}_\tau$ and rank content $c^{(i)}_\sigma$, respectively. The resulting policy gradient is
\begin{equation}
\label{eq:joint_gradient}
\nabla_\theta \mathcal{L}_{\mathrm{joint}} = -\frac{1}{G} \sum_{i=1}^{G} \frac{\hat{A}^{(i)}_{\mathrm{joint}}}{|o^{(i)}|} \biggl[\underbrace{\sum_{t \in \mathcal{T}^{(i)}_\tau} \nabla_\theta \log \pi_\theta(o^{(i)}_t \mid o^{(i)}_{<t})}_{\text{slate tokens}} + \underbrace{\sum_{t \in \mathcal{T}^{(i)}_\sigma} \nabla_\theta \log \pi_\theta(o^{(i)}_t \mid o^{(i)}_{<t})}_{\text{rank tokens}}\biggr],
\end{equation}
where $|o^{(i)}| = |\mathcal{T}^{(i)}_\tau| + |\mathcal{T}^{(i)}_\sigma|$.
Because a single $\hat{A}^{(i)}_{\mathrm{joint}}$ multiplies \emph{both} sums, the gradient direction for slate tokens is influenced by ranking quality and vice versa.
This creates a credit-assignment failure: consider a rollout with high ranking quality ($R^{(i)}_{\mathrm{rank}} \gg \bar{R}_{\mathrm{rank}}$) but poor coverage ($R^{(i)}_{\mathrm{slate}} < \bar{R}_{\mathrm{slate}}$).
The combined advantage may be positive, reinforcing the very slate tokens that failed to surface relevant candidates.
Conversely, excellent coverage paired with poor ranking yields a negative combined advantage that penalizes good candidate generation.
In general, whenever the two reward components are not perfectly correlated across rollouts, the joint advantage introduces cross-phase gradient contamination that conflates \emph{what} was proposed with \emph{how} it was ordered.
Such discordance is structural, not pathological: coverage is order-invariant and favors broad slates, while ranking quality is position-sensitive and rewards selective concentration. The two objectives inherently pull in different directions.

\begin{remark}[Phase-specific gradient weighting]
\label{rem:gradient_independence}
Although both losses update shared parameters $\theta$, the advantage weighting ensures that each phase's gradient is scaled solely by its own reward signal. At $\theta = \theta_{\mathrm{old}}$, the gradient decomposes as $\nabla_\theta \mathcal{L} = \nabla_\theta \mathcal{L}_{\mathrm{slate}} + \lambda \, \nabla_\theta \mathcal{L}_{\mathrm{rank}} + \beta_{\mathrm{KL}} \, \nabla_\theta D_{\mathrm{KL}}$, where $\nabla_\theta \mathcal{L}_{\mathrm{slate}}$ depends only on $\hat{A}^{(i)}_{\mathrm{slate}}$ and $\nabla_\theta \mathcal{L}_{\mathrm{rank}}$ depends only on $\hat{A}^{(i)}_{\mathrm{rank}}$ (see Appendix~\ref{app:theory} for the full derivation). This first-order separability eliminates the cross-phase gradient contamination identified in Eq.~\eqref{eq:joint_advantage}.
\end{remark}

%% file: contents/5_experiments.tex
\section{Experiments}

\subsection{Experimental Setup}

\paragraph{Tasks and datasets.}
\textbf{Sequential recommendation} (MovieLens~\citep{10.1145/2827872}, LastFM~\citep{Bertin-Mahieux2011}): select and rank items from a candidate set of 20.
\textbf{Multi-hop QA} (HotpotQA~\citep{yang-etal-2018-hotpotqa}, MuSiQue~\citep{trivedi-etal-2022-musique}): select and rank 2--4 gold evidence passages from 20 candidates.
Dataset statistics and preprocessing details are in Appendix~\ref{app:experiments}.

\paragraph{Models and training.}
We experiment with \emph{Qwen3-4B-Instruct-2507}~\citep{qwen3technicalreport} and \emph{Qwen3.5-2B}~\citep{qwen3.5}. We focus on the 2B--4B scale to enable thorough ablation under practical compute constraints. RL training is initialized from an SFT warm-start, except for Qwen3.5-2B on QA, which starts from the pretrained model. We use Dr.~GRPO~\citep{liu2025understanding} with $G{=}8$ rollouts per prompt and evaluate with greedy decoding ($T{=}0$). Full training and evaluation details, including SFT variants and hyperparameters, are provided in~\Cref{app:training_details,app:eval_metrics,app:prompts}.

\paragraph{Reward instantiation.}
\label{sec:constructor_instantiations}
For both tasks, we instantiate the ranking reward with NDCG@$k$:
\begin{equation}
\label{eq:ndcg_reward}
R_{\mathrm{rank}}^{\mathrm{NDCG}}(x, \tau, \sigma; k) = \frac{\mathrm{DCG@}k(\tau_\sigma)}{\mathrm{IDCG@}k(\tau)}, \quad
\mathrm{DCG@}k(\tau_\sigma) = \sum_{i=1}^{\min(k,n)} \frac{y_{\sigma(i)}}{\log_2(1 + i)},
\end{equation}
where $y_i \in \{0,1\}$ is the binary relevance of item $e_i$ and $\mathrm{IDCG@}k(\tau)$ is the maximum achievable DCG over all reorderings of $\tau$. The slate reward counts recalled gold items and normalizes the raw count to recall for \textbf{recommendation} and to F1 for \textbf{QA} before advantage computation. We set $\lambda = 1.0$ throughout and provide a $\lambda$ ablation in Section~\ref{sec:analysis}. The recommendation reward choice is analyzed in Section~\ref{sec:analysis}.

\paragraph{Baselines.}
We compare against three paradigms, each isolating a different factor: \emph{traditional} models calibrate against specialized architectures; \emph{LLM zero-shot} methods establish the pretrained floor; and \emph{LLM trained} methods ablate the contributions of RL and factorized credit assignment.
Decoupled SFT is the closest structural analog, with separately trained selector and ranker modules. GRPO is the direct ablation of factorized credit assignment: it uses the same training setup as F-GRPO, but replaces the phase-specific advantages with a single joint reward $R_{\mathrm{joint}} = R_{\mathrm{slate}} + \lambda R_{\mathrm{rank}}$ applied uniformly to all tokens.
Full baseline details are provided in Appendix~\ref{app:baselines}; GRPO, zero-shot prompt formats, and SFT variants are detailed in \Cref{app:llm_baselines,app:zeroshot_prompts,app:sft_details}.

\input{fig/results_rec}

\subsection{Sequential Recommendation}
\label{sec:results_rec}

We report Recall@$k$ and NDCG@$k$ for $k \in \{1, 3, 5\}$ on LastFM and MovieLens in Table~\ref{tab:rec_results} (Precision@$k$ and Hit@$k$ are deferred to Appendix~\ref{app:results}).

\textit{(i)~RL fine-tuning yields large gains over supervised methods.}
F-GRPO improves substantially over SFT on both datasets (e.g., LastFM Recall@3: +53.7\% relative; MovieLens Recall@3: +82.9\% relative), confirming that sequence-level reward optimization provides learning signal that token-level imitation cannot capture.

\textit{(ii)~Factorized credit assignment improves over GRPO.}
F-GRPO improves over GRPO, with the clearest gains at higher $k$, where the slate's broader coverage compounds with the ranker's ordering (e.g., a +10.6\% relative gain in LastFM Recall@5 over GRPO).

\textit{(iii)~F-GRPO is competitive with specialized sequential models.}
Traditional baselines (GRU4Rec, UniSRec) score over learned item embeddings in a closed set, whereas our LLM generates item names as free-form text. Despite this harder setting, the Qwen3-4B F-GRPO variant surpasses all traditional methods at every cutoff on MovieLens and at @3 and @5 on LastFM.

\textit{(iv)~Decoupled pipelines suffer from distribution mismatch.}
The full Decoupled~SFT variant underperforms single-model SFT across all metrics despite requiring two separately fine-tuned backbone checkpoints, one for selection and one for ranking. This underperformance arises because the ranker is trained on gold slates, creating a distribution shift at inference. F-GRPO avoids this entirely: the ranker conditions on the slate generator's own rollout output during training, so both phases co-adapt to each other's evolving behavior.

\begin{figure*}[ht]
  \centering
  \begin{subfigure}[t]{0.45\textwidth}
    \centering
    \includegraphics[width=\linewidth]{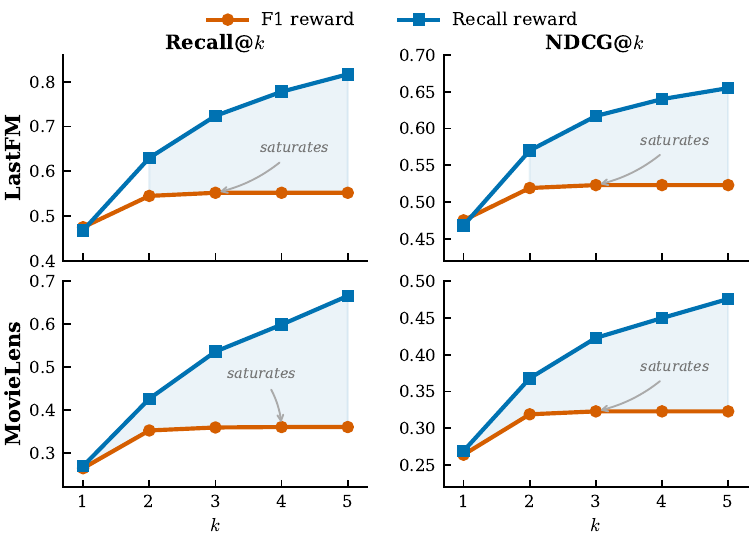}
    \caption{Slate reward ablation (Qwen3-4B).}
    \label{fig:reward_ablation}
  \end{subfigure}
  \hfill
  \begin{subfigure}[t]{0.54\textwidth}
    \centering
    \includegraphics[width=\linewidth]{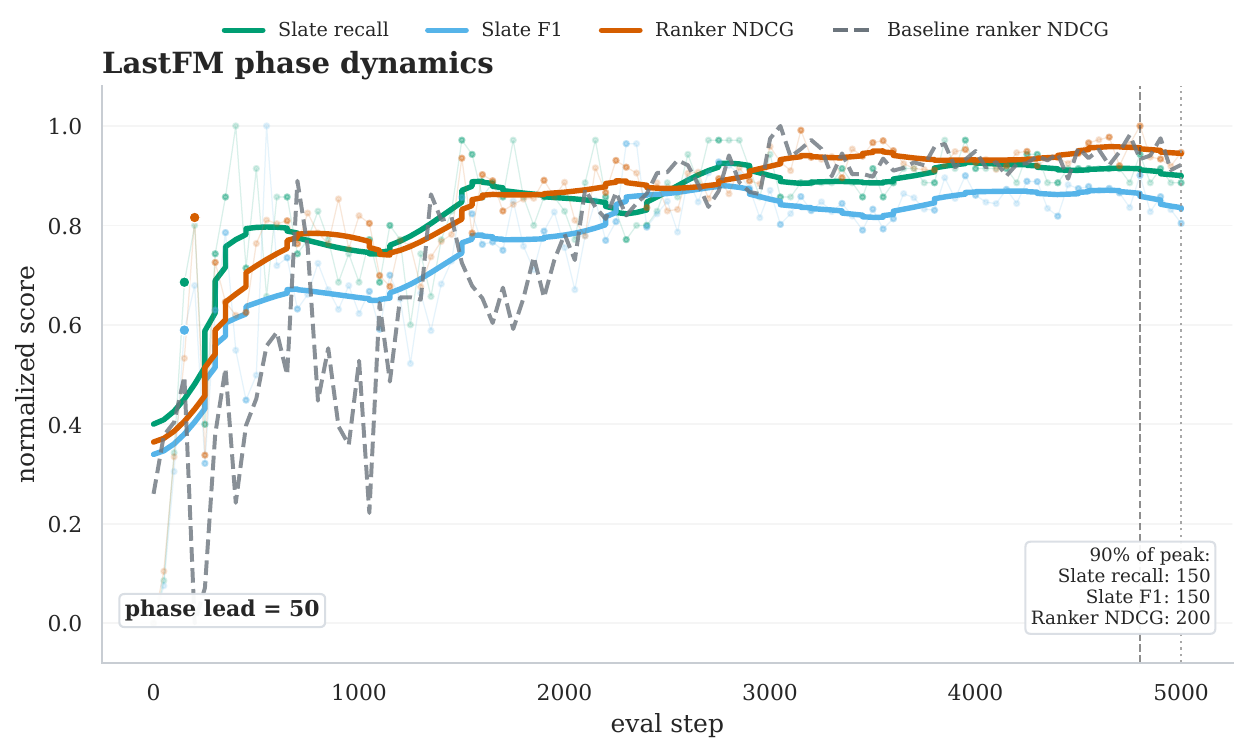}
    \caption{Phase separation on LastFM (Qwen3.5-2B).}
    \label{fig:rec_phase_dynamics}
  \end{subfigure}
  \caption{Training dynamics on LastFM. (a)~Slate reward ablation. (b)~The slate generator matures before the ranker.}
  \label{fig:training_dynamics}
\end{figure*}
\subsection{Multi-Hop Question Answering}
\label{sec:results_qa}
\input{fig/results_qa}

We report Recall@$k$ and NDCG@$k$ for $k \in \{1, 3, 5\}$ on MuSiQue and HotpotQA in Table~\ref{tab:qa_results} (Precision@$k$ and Hit@$k$ are deferred to Appendix~\ref{app:results}).

\textit{(i)~Factorized credit assignment helps most when coverage is the bottleneck.}
On MuSiQue, F-GRPO outperforms GRPO at both model scales, with gains concentrated at higher cutoffs (e.g., a +13.2\% relative gain in Recall@3 for Qwen3-4B).
On HotpotQA, the 4B models are essentially tied, while the 2B model shows a clear gap at @3 and above (a +5.9\% relative gain at Recall@3), suggesting phase-specific credit assignment is most beneficial when the backbone makes coverage harder.

\textit{(ii)~In-domain RL training competitive with dedicated rerankers.}
The reranker baselines are trained on MS~MARCO and applied zero-shot, whereas GRPO and F-GRPO are trained in-domain.
On HotpotQA, the F-GRPO 4B models outperform all dedicated rerankers despite using a general-purpose backbone rather than a reranking-specialized model.
On MuSiQue, F-GRPO exceeds MonoT5 at R@1 and @3, while being competitive at @5.

\textit{(iii)~Decoupled pipelines remain less reliable than factorized RL.}
The selector-only and full decoupled variants are consistently weak, and even the strongest variant (rank only) remains below F-GRPO at the cutoffs where coverage matters most.
This mirrors the recommendation setting: training the ranker on gold slates but deploying it on generated slates creates a distribution mismatch that end-to-end factorized RL avoids.
Representative HotpotQA outputs are provided in Appendix~\ref{app:qa_examples}.

\section{Analysis}
\label{sec:analysis}

\begin{figure*}[htp]
  \centering
  \includegraphics[width=\textwidth]{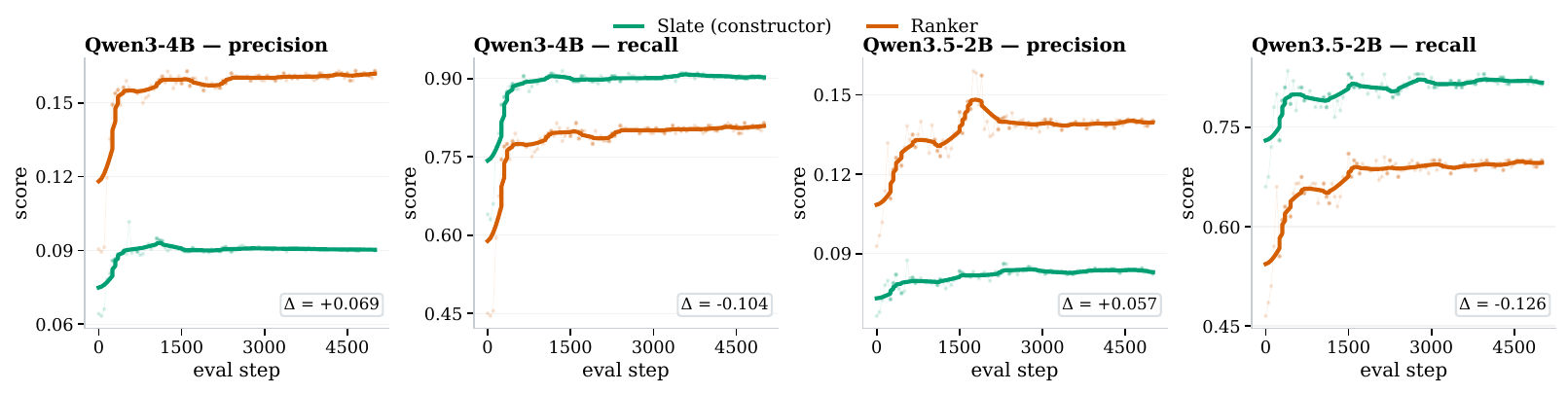}
  \caption{Precision--recall redistribution between the slate and ranker on LastFM across two model sizes. The ranker trades recall for precision by filtering the slate.}
  \label{fig:rec_precision_recall}
\end{figure*}
\paragraph{Slate reward formulation.}\Cref{fig:reward_ablation} compares F1- and recall-based slate rewards. At $k{=}1$, the two perform comparably, but the F1 reward saturates by $k{=}3$ on LastFM and $k{=}4$ on MovieLens, as its precision penalty discourages proposing candidates beyond the gold set. The recall reward improves monotonically, reaching a +48\% relative gain in Recall@5 on LastFM and +85\% on MovieLens. The main results use the recall-reward variant.

\paragraph{Phase separation.}\Cref{fig:rec_phase_dynamics} tracks when each component of the factorized policy matures during training. The slate generator reaches 90\% of its peak slate recall by step~150, while the ranker does not reach the same fraction of its peak NDCG until step~200. This ordering is consistent with the conditional structure of the factorized policy (Eq.~\eqref{eq:factorized}): because the ranker conditions on the slate ($\pi^{\mathrm{rank}}_\theta(\sigma \mid x, \tau)$), it cannot rank effectively until the slate generator provides a sufficiently informative candidate set.

\paragraph{Precision--recall redistribution.}\Cref{fig:rec_precision_recall} compares set-level precision and recall of the slate and ranker throughout training on LastFM across two model sizes.
The slate generator consistently achieves high recall (0.90 for Qwen3-4B, 0.82 for Qwen3.5-2B) but low precision, reflecting its role as a broad candidate generator. The ranker reverses this balance: at convergence, the ranker achieves precision 0.163 versus the slate's 0.090 for Qwen3-4B (0.140 vs.\ 0.082 for Qwen3.5-2B), while recall decreases correspondingly.
The pattern holds across both model scales, showing that factorized training consistently produces broad slate coverage followed by sharper top-of-list concentration.

Additional analyses of optimization dynamics and phase-specific error attribution are deferred to Appendix~\ref{app:optimization_dynamics} and~\ref{app:error_attribution}.

\subsection{Ablation Studies}
We conduct ablation studies on LastFM (Qwen3.5-2B) to isolate the contributions of each design choice. Full results and figures are in Appendix~\ref{app:ablation_results}. Sensitivity analysis shows that $\lambda{=}1.0$ and slate size $n{=}10$ are robust defaults (\Cref{fig:hyperparameter_ablation}). Underweighting the ranking loss with $\lambda{=}0.5$ consistently degrades top-position quality, while $n{=}5$ limits coverage and $n{=}15$ dilutes the pool. These trends are stable across Recall@$k$ and NDCG@$k$, supporting $\lambda{=}1.0$ and $n{=}10$ as effective defaults.

%% file: fig/results_rec.tex
\begin{table}[t]
  \centering
  \resizebox{\textwidth}{!}{%
  \begin{tabular}{@{}lcccccc@{\hspace{0.8em}}!{\vrule width 0.3pt}@{\hspace{0.8em}}cccccc@{}}
  \toprule
  & \multicolumn{6}{c}{LastFM} & \multicolumn{6}{c}{MovieLens} \\
  \cmidrule(lr){2-7} \cmidrule(lr){8-13}
  & \multicolumn{3}{c}{Recall@$k$} & \multicolumn{3}{c}{NDCG@$k$} & \multicolumn{3}{c}{Recall@$k$} & \multicolumn{3}{c}{NDCG@$k$} \\
  \cmidrule(lr){2-4} \cmidrule(lr){5-7} \cmidrule(lr){8-10} \cmidrule(lr){11-13}
  Method & @1 & @3 & @5 & @1 & @3 & @5 & @1 & @3 & @5 & @1 & @3 & @5 \\
  \midrule
  Random & 5.0 & 15.0 & 25.1 & 5.0 & 10.6 & 14.7 & 4.9 & 15.3 & 25.1 & 4.9 & 10.8 & 14.8 \\
  Popularity & 24.7 & 45.9 & 56.6 & 24.7 & 36.9 & 41.3 & 14.4 & 36.3 & 55.1 & 14.4 & 26.9 & 34.6 \\
  BM25 & 6.2 & 17.8 & 27.7 & 6.2 & 12.8 & 16.8 & 8.7 & 18.4 & 27.7 & 8.7 & 14.1 & 17.9 \\
  GRU4Rec & 50.8 & 69.0 & 75.5 & 50.8 & 61.3 & 64.0 & 21.8 & 46.8 & 60.2 & 21.8 & 36.1 & 41.6 \\
  UniSRec (zero-shot) & 11.0 & 25.4 & 36.9 & 11.0 & 19.2 & 23.9 & 6.3 & 16.5 & 25.6 & 6.3 & 12.0 & 15.8 \\
  UniSRec (fine-tuned) & 23.9 & 42.2 & 54.7 & 23.9 & 34.5 & 39.6 & 13.6 & 29.8 & 43.5 & 13.6 & 22.7 & 28.4 \\
  \addlinespace[3pt]
  \midrule
  \multicolumn{13}{@{}l}{\textbf{Qwen3-4B}} \\
  Zero-shot (two-step) & 18.8 & 44.7 & 58.4 & 18.8 & 34.0 & 39.6 & 12.1 & 29.0 & 40.0 & 12.1 & 21.6 & 26.1 \\
  Zero-shot (single-step) & 20.6 & 43.8 & 55.8 & 20.6 & 34.3 & 39.2 & 11.9 & 26.7 & 38.8 & 11.9 & 20.3 & 25.2 \\
  SFT & 40.1 & 47.1 & 53.8 & 40.1 & 44.0 & 46.8 & 21.5 & 29.3 & 36.5 & 21.5 & 25.8 & 28.8 \\
  \cdashline{1-13}\noalign{\vskip 2pt}
  Decoupled SFT (sel.\ only) & 14.1 & 33.3 & 46.0 & 14.1 & 25.3 & 30.5 & 10.1 & 25.3 & 36.0 & 10.1 & 18.9 & 23.3 \\
  Decoupled SFT (rank.\ only) & 28.8 & 42.9 & 51.9 & 28.8 & 37.0 & 40.7 & 16.5 & 28.0 & 37.0 & 16.5 & 23.0 & 26.8 \\
  Decoupled SFT (full) & 33.4 & 40.4 & 46.3 & 33.4 & 37.4 & 39.8 & 19.7 & 27.8 & 35.5 & 19.7 & 24.3 & 27.5 \\
  \cdashline{1-13}\noalign{\vskip 2pt}
  GRPO & \underline{44.7} & \underline{55.1} & \underline{73.9} & \underline{44.7} & \underline{50.8} & \underline{58.6} & \underline{23.9} & \underline{44.4} & \underline{57.7} & \underline{23.9} & \underline{35.8} & \underline{41.2} \\
  \textbf{F-GRPO (ours)} & \textbf{46.8} & \textbf{72.4} & \textbf{81.7} & \textbf{46.8} & \textbf{61.7} & \textbf{65.5} & \textbf{26.9} & \textbf{53.6} & \textbf{66.6} & \textbf{26.9} & \textbf{42.3} & \textbf{47.6} \\
  \midrule
  \addlinespace[3pt]
  \multicolumn{13}{@{}l}{\textbf{Qwen3.5-2B}} \\
  Zero-shot (two-step) & 0.6 & 1.6 & 2.3 & 0.6 & 1.1 & 1.5 & 1.2 & 1.8 & 2.2 & 1.2 & 1.5 & 1.7 \\
  Zero-shot (single-step) & 3.1 & 6.5 & 8.1 & 3.1 & 5.1 & 5.7 & 3.3 & 4.7 & 5.5 & 3.3 & 4.0 & 4.4 \\
  SFT & 37.9 & 45.4 & 52.6 & 37.9 & 42.1 & 45.1 & \textbf{27.7} & 34.7 & 41.0 & \textbf{27.7} & 31.6 & 34.2 \\
  \cdashline{1-13}\noalign{\vskip 2pt}
  Decoupled SFT (sel.\ only) & 0.5 & 1.3 & 2.0 & 0.5 & 0.9 & 1.2 & 0.7 & 1.8 & 2.7 & 0.7 & 1.3 & 1.7 \\
  Decoupled SFT (rank.\ only) & 23.1 & 31.6 & 38.3 & 23.1 & 28.0 & 30.8 & 12.3 & 20.7 & 27.1 & 12.3 & 17.1 & 19.7 \\
  Decoupled SFT (full) & 32.4 & 40.1 & 46.1 & 32.4 & 36.9 & 39.3 & 23.6 & 31.6 & 40.6 & 23.6 & 28.2 & 31.9 \\
  \cdashline{1-13}\noalign{\vskip 2pt}
  GRPO & \underline{38.7} & \underline{60.4} & \underline{71.3} & \underline{38.7} & \underline{51.3} & \underline{55.8} & \underline{24.4} & \textbf{50.7} & \underline{56.4} & \underline{24.4} & \textbf{39.6} & \textbf{42.1} \\
  \textbf{F-GRPO (ours)} & \textbf{40.0} & \textbf{63.2} & \textbf{74.3} & \textbf{40.0} & \textbf{53.5} & \textbf{58.1} & 21.5 & \underline{45.8} & \textbf{61.1} & 21.5 & \underline{35.5} & \underline{41.8} \\
  \bottomrule
  \end{tabular}%
  }
  \caption{Recall@$k$ and NDCG@$k$ on LastFM and MovieLens for Qwen3-4B and Qwen3.5-2B. All values are percentages. \textbf{Bold}: best among LLM methods; \underline{Underlined}: second best.}
  \label{tab:rec_results}
\end{table}

%% file: fig/results_qa.tex
\begin{table}[t]
  \centering
  \resizebox{\textwidth}{!}{%
  \begin{tabular}{@{}lcccccc@{\hspace{0.8em}}!{\vrule width 0.3pt}@{\hspace{0.8em}}cccccc@{}}
  \toprule
  & \multicolumn{6}{c}{MuSiQue} & \multicolumn{6}{c}{HotpotQA} \\
  \cmidrule(lr){2-7} \cmidrule(lr){8-13}
  & \multicolumn{3}{c}{Recall@$k$} & \multicolumn{3}{c}{NDCG@$k$} & \multicolumn{3}{c}{Recall@$k$} & \multicolumn{3}{c}{NDCG@$k$} \\
  \cmidrule(lr){2-4} \cmidrule(lr){5-7} \cmidrule(lr){8-10} \cmidrule(lr){11-13}
  Method & @1 & @3 & @5 & @1 & @3 & @5 & @1 & @3 & @5 & @1 & @3 & @5 \\
  \midrule
  Random & 5.0 & 15.0 & 24.7 & 13.3 & 14.9 & 19.6 & 5.0 & 15.0 & 25.1 & 9.9 & 13.1 & 18.1 \\
  BM25 & 22.9 & 42.9 & 53.4 & 57.0 & 46.7 & 51.4 & 41.8 & 72.2 & 82.7 & 83.6 & 73.2 & 78.5 \\
  RankZephyr & 33.5 & 61.3 & 71.5 & 82.3 & 67.1 & 71.1 & 47.4 & 84.9 & 90.3 & 94.8 & 85.9 & 88.6 \\
  MonoT5 & 34.9 & 64.5 & 74.9 & 86.1 & 70.7 & 74.6 & 46.3 & 86.2 & 92.8 & 92.5 & 86.2 & 89.5 \\
  DuoT5 & 34.3 & 62.8 & 73.5 & 84.8 & 69.2 & 73.3 & 43.2 & 79.8 & 86.2 & 86.4 & 80.0 & 83.3 \\
  LiT5 & 33.6 & 54.7 & 64.4 & 82.7 & 61.9 & 65.6 & 46.3 & 78.4 & 85.2 & 92.5 & 80.3 & 83.7 \\
  \addlinespace[3pt]
  \midrule
  \multicolumn{13}{@{}l}{\textbf{Qwen3-4B}} \\
  Zero-shot (two-step)         & 4.1 & 8.5 & 9.0 & 10.2 & 9.1 & 9.2 & 3.0 & 5.8 & 6.1 & 6.0 & 5.8 & 5.9 \\
  Zero-shot (single-step)      & 23.4 & 48.1 & 56.3 & 58.4 & 51.6 & 54.7 & 30.7 & 61.0 & 65.4 & 61.4 & 60.2 & 62.4 \\
  SFT                          & 27.9 & 54.4 & 60.9 & 69.5 & 59.5 & 61.5 & 38.3 & 68.7 & 73.0 & 76.6 & 69.6 & 71.8 \\
  \cdashline{1-13}\noalign{\vskip 2pt}
  Decoupled SFT (sel.\ only)   & 8.6 & 16.3 & 18.3 & 21.7 & 18.0 & 18.5 & 9.8 & 19.1 & 20.6 & 19.6 & 18.9 & 19.7 \\
  Decoupled SFT (rank.\ only)  & 32.3 & \underline{63.0} & \underline{68.3} & 79.8 & 68.6 & \underline{69.7} & \underline{43.8} & 84.3 & 89.2 & 87.6 & 84.0 & 86.5 \\
  Decoupled SFT (full)         & 24.0 & 39.0 & 44.2 & 60.6 & 44.9 & 46.5 & 32.6 & 49.0 & 52.9 & 65.1 & 52.0 & 54.0 \\
  \cdashline{1-13}\noalign{\vskip 2pt}
  GRPO                         & \underline{36.4} & \underline{63.0} & 63.0 & \underline{90.8} & \underline{71.3} & 69.2 & \textbf{48.0} & \textbf{93.0} & \textbf{93.1} & \textbf{96.1} & \underline{93.0} & \underline{93.1} \\
  \textbf{F-GRPO (ours)}       & \textbf{36.8} & \textbf{71.3} & \textbf{71.8} & \textbf{91.6} & \textbf{78.5} & \textbf{76.5} & \textbf{48.0} & \underline{92.6} & \underline{92.6} & \underline{96.0} & \textbf{93.3} & \textbf{93.3} \\
  \midrule
  \addlinespace[3pt]
  \multicolumn{13}{@{}l}{\textbf{Qwen3.5-2B}} \\
  Zero-shot (two-step)         & 2.2 & 4.1 & 4.8 & 5.4 & 4.5 & 4.7 & 1.0 & 1.9 & 2.1 & 2.0 & 1.9 & 2.0 \\
  Zero-shot (single-step)      & 5.2 & 10.6 & 12.4 & 13.2 & 11.5 & 12.1 & 4.0 & 7.4 & 8.4 & 8.0 & 7.4 & 7.9 \\
  SFT                          & 13.5 & 29.2 & 38.0 & 34.4 & 31.2 & 35.0 & 5.2 & 15.1 & 25.3 & 10.4 & 13.2 & 18.2 \\
  \cdashline{1-13}\noalign{\vskip 2pt}
  Decoupled SFT (sel.\ only)   & 0.3 & 0.6 & 0.7 & 0.7 & 0.6 & 0.7 & 0.3 & 0.6 & 0.7 & 0.5 & 0.5 & 0.6 \\
  Decoupled SFT (rank.\ only)  & 17.8 & 37.4 & 49.8 & 43.7 & 39.0 & 44.8 & \underline{21.4} & 51.7 & 68.5 & 42.8 & 47.7 & 56.1 \\
  Decoupled SFT (full)         & 5.2 & 14.8 & 24.5 & 13.8 & 14.7 & 19.5 & 6.0 & 16.6 & 25.7 & 11.9 & 14.8 & 19.3 \\
  \cdashline{1-13}\noalign{\vskip 2pt}
  GRPO                         & \underline{36.1} & \underline{63.1} & \underline{65.2} & \textbf{90.0} & \underline{71.4} & \underline{70.6} & \textbf{48.1} & \underline{86.3} & \underline{86.5} & \underline{96.1} & \underline{88.4} & \underline{88.5} \\
  \textbf{F-GRPO (ours)}       & \textbf{36.2} & \textbf{72.7} & \textbf{73.2} & \underline{89.9} & \textbf{79.2} & \textbf{77.2} & \textbf{48.1} & \textbf{91.4} & \textbf{91.4} & \textbf{96.3} & \textbf{92.5} & \textbf{92.5} \\
  \bottomrule
  \end{tabular}%
  }
  \caption{Recall@$k$ and NDCG@$k$ on MuSiQue and HotpotQA for Qwen3-4B and Qwen3.5-2B. All values are percentages. \textbf{Bold}: best among LLM methods; \underline{Underlined}: second best.}
  \label{tab:qa_results}
\end{table}

%% file: contents/2_related_new.tex
\section{Related Work}
\label{sec:related}

\subsection{LLM-based retrieval and ranking.}
LLMs have been applied to retrieval and ranking across several paradigms.
Generative retrieval methods emit document identifiers directly~\citep{bevilacqua2022autoregressive,wang2022neural,10.1145/3653712}.
When a candidate pool is available, LLMs serve as rerankers: pointwise~\citep{nogueira-etal-2020-document}, pairwise~\citep{pradeep2021expandomonoduodesignpatterntext}, or listwise~\citep{sun-etal-2023-chatgpt,pradeep2023rankzephyreffectiverobustzeroshot,tamber2023scalingdownlittingup}.
More recent work optimizes ranking with RL: Neural PG-RANK~\citep{gao2024policygradienttraininglanguagemodels} trains a Plackett-Luce ranking policy with REINFORCE, Rank-R1~\citep{zhuang2025rankr1enhancingreasoningllmbased} applies GRPO to setwise reranking, Rank-GRPO~\citep{zhu2026rankgrpotrainingllmbasedconversational} elevates credit assignment to the rank level, and IRPO~\citep{wu2025incontext} and~\citet{huang2026listwise} extend DPO to listwise preferences.
Closest to our work are methods that couple generation and ranking: two-stage counterfactual learning~\citep{10.1145/3731120.3744583}, GeMS~\citep{10.1145/3539597.3570412}, and HiGR~\citep{pang2026higrefficientgenerativeslate}.
However, these approaches either separate the generator and ranker into distinct models or train both stages under a single shared objective, leaving no mechanism to attribute outcome quality to individual phases.
F-GRPO instead generates and ranks within a single autoregressive rollout and applies phase-specific credit assignment, so the slate generator and ranker each receive a gradient signal from their own reward.

\subsection{Reinforcement learning for LLMs.}
PPO~\citep{schulman2017proximalpolicyoptimizationalgorithms} was scaled to LLM alignment in RLHF~\citep{ouyang2022training}; subsequent work favors critic-free variants.
GRPO~\citep{shao2024deepseekmathpushinglimitsmathematical} replaces the critic with group-relative normalization and has been widely adopted for large-scale reasoning-oriented RL training.
Refinements include Dr.~GRPO~\citep{liu2025understanding} (unbiased normalization), DAPO~\citep{yu2025dapo} (decoupled clipping and token-level loss), and GVPO~\citep{zhang2025gvpogroupvariancepolicy} (provably optimal group weights).
Offline alternatives such as DPO~\citep{10.5555/3666122.3668460,wang2026scenealign,huang2025image,li2025importance,huang2025pluralistic,kveton2025active} optimize directly from pairwise preferences.
These methods target monolithic completions with a single reward; our work extends GRPO to structured rollouts~\citep{yu2025explainable,wuctrls,wu2025ocean,surana2026generatefiltercontrolreplay,wu2024decot} containing two coupled phases, each requiring its own credit assignment.
An extended discussion is in Appendix~\ref{app:related_extended}.

%% file: contents/6_conclusion.tex
\section{Conclusion}

We introduced F-GRPO, a factorized policy optimization framework that trains a single LLM to jointly generate and rank candidates by computing separate group-relative advantages for each phase.
This eliminates the gradient interference of single-reward GRPO and the distribution mismatch of decoupled pipelines, while our theoretical analysis formalizes the first-order separability of the two credit signals.
Experiments across sequential recommendation and multi-hop QA show improvements over the GRPO and decoupled baselines, with the strongest gains appearing when proposal coverage is the main bottleneck.

\textbf{LLM usage:} AI writing tools were used to assist in drafting and verifying the theoretical proofs in this paper.

%% file: contents/7_appendix.tex
\section{Theoretical Analysis}
\label{app:theory}

\subsection{Gradient Derivation for Two-Phase Sequence-Level Loss}

We derive the gradient of the two-phase loss (Eq.~\eqref{eq:combined_loss}) and show that, at the expansion point $\theta = \theta_{\mathrm{old}}$, it decomposes into two GRPO updates, each weighted solely by its own phase-specific reward signal. This formalizes the phase-specific gradient weighting claimed in Remark~\ref{rem:gradient_independence}.

\begin{proposition}[Gradient decomposition]
\label{prop:gradient_decomposition}
At $\theta = \theta_{\mathrm{old}}$, the gradient of the combined loss decomposes as
$\nabla_\theta \mathcal{L} = \nabla_\theta \mathcal{L}_{\mathrm{slate}} + \lambda \, \nabla_\theta \mathcal{L}_{\mathrm{rank}} + \beta_{\mathrm{KL}} \, \nabla_\theta D_{\mathrm{KL}}$,
where $\nabla_\theta \mathcal{L}_{\mathrm{slate}}$ depends only on the slate advantages $\hat{A}^{(i)}_{\mathrm{slate}}$ and $\nabla_\theta \mathcal{L}_{\mathrm{rank}}$ depends only on the rank advantages $\hat{A}^{(i)}_{\mathrm{rank}}$. At this expansion point, no cross-phase reward information enters either gradient term.
\end{proposition}

\begin{proof}
For each rollout $i$, recall that $c^{(i)}_\tau$ and $c^{(i)}_\sigma$ denote the slate and rank content tokens, respectively.
The two-phase loss from Eq.~\eqref{eq:combined_loss} is:
\begin{equation}
\mathcal{L}(\theta) = \mathcal{L}_{\mathrm{slate}} + \lambda \, \mathcal{L}_{\mathrm{rank}} + \beta_{\mathrm{KL}} \, D_{\mathrm{KL}}(\pi_\theta \| \pi_{\mathrm{ref}}).
\end{equation}

\paragraph{Slate Phase Gradient.}
The slate loss (Eq.~\ref{eq:slate_loss}) applies the clipped surrogate with per-token ratios $\rho^{(i)}_{\tau,t} = \pi_\theta(c^{(i)}_{\tau,t} \mid x, c^{(i)}_{\tau,<t}) / \pi_{\theta_{\mathrm{old}}}(c^{(i)}_{\tau,t} \mid x, c^{(i)}_{\tau,<t})$ and a uniform advantage $\hat{A}^{(i)}_{\mathrm{slate}}$ across all slate tokens:
\begin{equation}
\mathcal{L}_{\mathrm{slate}} = -\frac{1}{G} \sum_{i=1}^{G} \frac{1}{|c^{(i)}_\tau|} \sum_{t=1}^{|c^{(i)}_\tau|}
\min\!\bigl( \rho^{(i)}_{\tau,t}\, \hat{A}^{(i)}_{\mathrm{slate}},\; \mathrm{clip}(\rho^{(i)}_{\tau,t},\, 1{-}\epsilon_{\mathrm{clip}},\, 1{+}\epsilon_{\mathrm{clip}})\, \hat{A}^{(i)}_{\mathrm{slate}} \bigr).
\end{equation}

At $\theta = \theta_{\mathrm{old}}$, all ratios $\rho^{(i)}_{\tau,t} = 1$, so the clipping is inactive and the gradient simplifies to:
\begin{align}
\nabla_\theta \mathcal{L}_{\mathrm{slate}} \big|_{\theta=\theta_{\mathrm{old}}}
&= -\frac{1}{G} \sum_{i=1}^{G} \frac{\hat{A}^{(i)}_{\mathrm{slate}}}{|c^{(i)}_\tau|} \sum_{t=1}^{|c^{(i)}_\tau|} \nabla_\theta \log \pi_\theta(c^{(i)}_{\tau,t} \mid x, c^{(i)}_{\tau,<t}).
\end{align}
This is the standard REINFORCE gradient with phase-specific group-relative advantages applied uniformly to all tokens in the slate, with per-rollout length normalization.

\paragraph{Rank Phase Gradient.}
Analogously, the rank loss (Eq.~\ref{eq:rank_loss}) uses per-token ratios $\rho^{(i)}_{\sigma,t}$ conditioned on the prompt augmented with the generated slate $(x, \tau^{(i)})$:
\begin{align}
\nabla_\theta \mathcal{L}_{\mathrm{rank}} \big|_{\theta=\theta_{\mathrm{old}}}
&= -\frac{1}{G} \sum_{i=1}^{G} \frac{\hat{A}^{(i)}_{\mathrm{rank}}}{|c^{(i)}_\sigma|} \sum_{t=1}^{|c^{(i)}_\sigma|} \nabla_\theta \log \pi_\theta(c^{(i)}_{\sigma,t} \mid x, \tau^{(i)}, c^{(i)}_{\sigma,<t}).
\end{align}

\paragraph{Independence of reward signals.}
The slate gradient is weighted solely by $\hat{A}^{(i)}_{\mathrm{slate}}$, computed from $R_{\mathrm{slate}}$ (Section~\ref{sec:problem}), and the rank gradient is weighted solely by $\hat{A}^{(i)}_{\mathrm{rank}}$, computed from $R_{\mathrm{rank}}$ (Section~\ref{sec:problem}).
Although both gradients update shared parameters $\theta$, the advantage weighting ensures that each phase's update direction is determined by its own reward alone.
The slate is optimized purely for coverage quality and the ranker purely for ordering quality, with no cross-phase interference.
This completes the proof.
\end{proof}

\paragraph{Relation to standard GRPO.}
Each phase independently instantiates the GRPO objective from \citet{shao2024deepseekmathpushinglimitsmathematical} (Eq.~\ref{eq:grpo_objective}).
The combined loss $\mathcal{L} = \mathcal{L}_{\mathrm{slate}} + \lambda \mathcal{L}_{\mathrm{rank}}$ is therefore equivalent to running two GRPO optimizations on shared model parameters, where $\lambda$ controls the relative importance of ranking quality.
The key difference from standard GRPO is that each optimization operates on a \emph{different token subsequence} of the same rollout and uses a \emph{different reward function}, whereas standard GRPO applies a single reward uniformly to all tokens.

\input{fig/algo}
\section{Dataset Details}
\label{app:experiments}

\subsection{Dataset Statistics}
\begin{table}[h]
\centering
\begin{tabular}{lcccc}
\toprule
Dataset & Train & Val & Test & Gold Count \\
\midrule
\multicolumn{5}{c}{\textit{Recommendation}} \\
MovieLens & 10,000 & 200 & 2,000 & 1 item / 20 candidates \\
LastFM & 10,000 & 200 & 2,000 & 1 item / 20 candidates \\
\midrule
\multicolumn{5}{c}{\textit{Multi-Hop Question Answering}} \\
HotpotQA & 10,000 & 200 & 2,000 & 2 passages / 20 candidates \\
MuSiQue & 10,000 & 200 & 2,000 & 2--4 passages / 20 candidates \\
\bottomrule
\end{tabular}
\caption{Dataset statistics. Counts reflect subsampled splits used in experiments.}
\label{tab:dataset_stats}
\end{table}

\subsection{Dataset Formulation}
\label{app:dataset_details}

We describe how each dataset is adapted for factorized list-and-rank training. All datasets follow the same structure: a context $x$, a candidate pool of 20 items, binary relevance labels, and a task-specific prompt. Candidates are shuffled to prevent positional shortcuts.

\paragraph{MovieLens.}
Each sample provides a user history of 10 previously watched movies and a candidate pool of 20 movies (1 true next movie + 19 distractors). Movie titles are normalized to handle year suffixes (e.g., ``Toy Story (1995)'' $\to$ ``Toy Story'').

\paragraph{LastFM.}
Each sample provides a variable-length user history (18--220 artists, mean ${\sim}125$) and 20 candidates (1 gold + 19 distractors). The prompt template is identical to MovieLens, with ``items'' referring to artist names instead of movie titles.

\paragraph{HotpotQA.}
Each sample provides a multi-hop question and 20 candidate passages (exactly 2 gold supporting passages + 18 distractors), constructed from the HotpotQA distractor split~\citep{yang-etal-2018-hotpotqa}. Passages are identified by abstract IDs (e.g., P5, P15) that serve the same role as item names in recommendation. The task requires selecting and ranking the two bridge passages needed to answer the question.

\paragraph{MuSiQue.}
Each sample provides a compositional multi-hop question~\citep{trivedi-etal-2022-musique} and 20 candidate passages (2--4 gold supporting passages + 16--18 distractors). Passages are shuffled and assigned abstract IDs (P1--P20). MuSiQue provides official train (19,938) and validation (2,417) splits but no separate test split. In our experiments, we subsample 10,000 training examples from the train split and construct disjoint validation and test subsets from the original validation split: 200 examples for validation and 2,000 for testing.

\section{Implementation Details}
\label{app:impl_details}

This section describes the key implementation choices that bridge the formal objective (Section~\ref{sec:method}) to a working training pipeline.

\subsection{Structured Output Format}

All tasks use the same structured output format with XML-style delimiter tags.
The model generates two segments in a single autoregressive pass:

\paragraph{Recommendation.}
Items are full names (movie/artist titles); the rank reorders by predicted preference:
\begin{verbatim}
<SLATE>Item A, Item B, Item C, Item D</SLATE>
<RANK>Item C, Item A, Item D, Item B</RANK>
\end{verbatim}

\paragraph{Question answering.}
Items are passage IDs; the rank reorders by relevance to the query:
\begin{verbatim}
<SLATE>P2, P5, P8</SLATE>
<RANK>P5, P2, P8</RANK>
\end{verbatim}

During parsing, we extract content within tags and apply task-specific normalization for reward computation.

\subsection{Delimiter Masking}
\label{app:delimiter_masking}

The two-phase loss (Eq.~\eqref{eq:combined_loss}) requires that delimiter tokens participate in the forward pass (so that content tokens are conditioned on the correct prefix) but are excluded from the loss.
We implement this via \emph{position-based masking}: for each delimiter pair (e.g., \texttt{<SLATE>}, \texttt{</SLATE>}), we cache the opening token count $N_{\mathrm{open}}$ and closing token count $N_{\mathrm{close}}$ at initialization and zero out the loss for the first $N_{\mathrm{open}}$ and last $N_{\mathrm{close}}$ positions of each segment.
Letting $T_{\mathrm{seg}}$ denote the total number of tokens in a segment (including delimiters), the effective segment loss is $L_{\mathrm{seg}} = \sum_{t=N_{\mathrm{open}}+1}^{T_{\mathrm{seg}}-N_{\mathrm{close}}} \ell_t$.

\subsection{Malformed Output Handling}
\label{app:three_case}

During training, the model may produce outputs missing one or both delimiter tags. We classify each rollout into one of three cases and assign rewards accordingly:

\begin{enumerate}
    \item \textbf{Case 1} (both \texttt{<SLATE>} and \texttt{<RANK>} present): Normal two-phase processing with computed rewards for both phases.
    \item \textbf{Case 2} (only \texttt{<SLATE>} present): Slate reward is computed normally; the rank phase receives a format penalty $p_{\mathrm{fmt}} = -1.0$, applied to the raw continuation after \texttt{</SLATE>} (or to the full completion if generation stops immediately).
    \item \textbf{Case 3} (\texttt{<SLATE>} missing): Both phases receive $p_{\mathrm{fmt}}$, and the slate loss is applied to the full raw completion without delimiter masking.
\end{enumerate}

This classification determines the masking configuration for each rollout. Rollouts within the same case share the same delimiter configuration and are batched together. In practice, after SFT warm-start, nearly all rollouts are Case~1.

\subsection{Rank Subset Validation}

The ranker should reorder items from its own slate; introducing new items would violate the factorized conditioning $\pi^{\mathrm{rank}}_\theta(\sigma \mid x, \tau)$. We enforce this by checking whether every ranked item appears in the unique slate set $\mathrm{uniq}(\tau)$. If not, the rank reward is replaced with the format penalty $p_{\mathrm{fmt}}$, providing a negative learning signal against out-of-slate hallucination.

\section{Baseline Details}
\label{app:baselines}

We provide full descriptions of all baselines to facilitate reproduction. All baselines operate on the same candidate pools, data splits, and evaluation metrics as F-GRPO.

\subsection{Recommendation Baselines}
\label{app:rec_baselines}

All recommendation baselines operate on the same 20-candidate ranking task (1 gold + 19 distractors per sample) with identical evaluation (10,000 train / 2,000 test; 200 validation samples held out for model selection).

\paragraph{Random.}
Candidates are uniformly permuted and the top-5 returned. We average over 5 seeds for stable estimates.

\paragraph{Popularity.}
Item frequencies are counted from training answers with Laplace smoothing: $\mathrm{score}(i) = (\mathrm{count}(i) + 1) / (\mathrm{total} + |\mathcal{V}|)$, where $|\mathcal{V}|$ is the number of unique items in the training set.
Candidates are ranked by score. This non-personalized baseline measures signal in item popularity alone.

\paragraph{BM25~\citep{10.1561/1500000019}.}
The user's interaction history is concatenated into a query string. Each candidate item name is treated as a document and scored with BM25Okapi.
This tests whether surface-level lexical overlap between history and candidate names is informative.

\paragraph{GRU4Rec~\citep{hidasi2016sessionbasedrecommendationsrecurrentneural}.}
GRU-based sequential model following the official implementation.
Architecture: constrained embedding (shared input/output weights), 1-layer GRU (100 hidden).
Trained for 100 epochs (Adagrad, lr 0.05, batch 64) with cross-entropy loss and popularity-weighted negative sampling ($\alpha{=}0.5$).
Best checkpoint selected by validation NDCG@5 every 10 epochs.
Adapted from session-parallel batching to standard batching (our samples are independent histories).

\paragraph{UniSRec~\citep{10.1145/3534678.3539381}.}
Uses BERT-base CLS token embeddings (768D) as item representations, processed through a Mixture-of-Experts adaptor (8 experts, 768$\to$300) and a 2-layer SASRec-style encoder.
\emph{Zero-shot}: The pretrained checkpoint (\texttt{UniSRec-FHCKM-300}, trained on 5 Amazon domains) is applied directly without fine-tuning.
\emph{Fine-tuned}: The MoE adaptor is updated via inductive fine-tuning on the target domain (encoder frozen), trained for 300 epochs with early stopping (patience 10) on validation NDCG@10.
Data is converted to RecBole format for fine-tuning.

\subsection{QA Reranking Baselines}
\label{app:qa_baselines}

For the QA tasks (HotpotQA, MuSiQue), we include four reranking-specialized baselines, all applied \emph{zero-shot} (no task-specific training). These baselines are rerankers only: they reorder all given passages without an explicit selection phase. In contrast, our factorized approach first selects a subset via \texttt{<SLATE>}, then ranks within it via \texttt{<RANK>}. We report only ranker metrics for all rerankers (no slate metrics). All models are accessed via the \texttt{rank\_llm} library~\citep{10.1145/3726302.3730331}~(v0.21.0).

\paragraph{RankZephyr~\citep{pradeep2023rankzephyreffectiverobustzeroshot}.}
\label{app:rankzephyr}
A 7B model built on Zephyr-7B-$\beta$ (Mistral-based), fine-tuned for listwise reranking via knowledge distillation from GPT-3.5 and GPT-4 on 100K MS MARCO queries. All candidate passages are formatted into a listwise prompt template, and the model outputs a reordered list of identifiers (e.g., \texttt{[2] > [5] > [1] > ...}). Since both candidate pools fit within the model's window size, each query is reranked in a single forward pass. Checkpoint: \texttt{castorini/rank\_zephyr\_7b\_v1\_full}.

\paragraph{MonoT5~\citep{nogueira-etal-2020-document}.}
\label{app:monot5}
A 3B T5-based model fine-tuned on MS MARCO with a pointwise relevance objective: given a query-passage pair, the model generates ``true'' or ``false,'' and the softmax probability of ``true'' serves as the relevance score. Each candidate is scored independently ($N$ forward passes per query), then passages are ranked by descending score. Checkpoint: \texttt{castorini/monot5-3b-msmarco-10k}; context size 512, batch size 32.

\paragraph{DuoT5~\citep{pradeep2021expandomonoduodesignpatterntext}.}
\label{app:duot5}
A 3B T5-based model from the Expando-Mono-Duo framework, fine-tuned on MS MARCO with a pairwise relevance objective. For each query, all $\binom{N}{2} = 190$ candidate pairs are compared ($N{=}20$ for both datasets), and the final ranking is by descending total wins. Checkpoint: \texttt{castorini/duot5-3b-msmarco-10k}; context size 512.

\paragraph{LiT5~\citep{tamber2023scalingdownlittingup}.}
\label{app:lit5}
A T5-based model (${\sim}$3B parameters) that uses Fusion-in-Decoder (FiD) for listwise reranking: all candidates are encoded independently and fused during decoding, enabling joint comparison when producing the output permutation. The Distill variant is distilled from GPT-3.5/4 on MS MARCO. Checkpoint: \texttt{castorini/LiT5-Distill-xl}; context size 300, window size 20.

\paragraph{Summary of reranking paradigms.}
Table~\ref{tab:reranker_comparison} summarizes the key differences among the four rerankers.

\begin{table}[h]
\centering
\small
\begin{tabular}{lcccc}
\toprule
Method & Paradigm & Size & Complexity & Training \\
\midrule
MonoT5 & Pointwise & 3B & $O(N)$ passes & MS MARCO \\
DuoT5 & Pairwise & 3B & $O(N^2)$ passes & MS MARCO \\
RankZephyr & Listwise & 7B & 1 pass & MS MARCO (distilled) \\
LiT5 & Listwise (FiD) & 3B & 1 pass & MS MARCO (distilled) \\
\midrule
GRPO & Generative & 4B & 1 pass & In-domain (RL) \\
F-GRPO (ours) & Generative (factorized) & 4B & 1 pass & In-domain (RL) \\
\bottomrule
\end{tabular}
\caption{Comparison of QA reranking methods. $N$ is the number of candidate passages. Dedicated rerankers are trained on MS MARCO and applied zero-shot; our methods are trained in-domain with RL.}
\label{tab:reranker_comparison}
\end{table}

\subsection{LLM Baselines}
\label{app:llm_baselines}

\paragraph{Recommendation-side baselines.}
For recommendation, the compared LLM baselines correspond to prior paradigms: Zero-shot (two-step) follows staged prompting that separates candidate generation from ranking~\citep{wang2023zeroshotnextitemrecommendationusing}; Zero-shot (single-step) treats the model as a direct black-box generator/ranker~\citep{10.1007/978-3-031-56060-6_24}; the baseline SFT follows supervised generative recommendation setups~\citep{10.1145/3705728}; Decoupled SFT mirrors two-stage pipelines with separately trained selector and ranker modules~\citep{yue2023llamarectwostagerecommendationusing}; and GRPO uses the Dr.~GRPO variant~\citep{liu2025understanding}.

\paragraph{GRPO baseline.}
Identical training setup to F-GRPO (same SFT warm-start, same $G{=}8$ rollouts, same Dr.~GRPO normalization) but uses a single combined reward $R^{(i)}_{\mathrm{joint}} = R^{(i)}_{\mathrm{slate}} + \lambda\,R^{(i)}_{\mathrm{rank}}$ applied uniformly to all tokens.
This is the direct ablation of factorized credit assignment: everything is the same except the advantage computation collapses both phases into one signal.
Zero-shot and SFT baselines are described in Appendix~\ref{app:zeroshot_prompts} and~\ref{app:sft_details}, respectively.

\section{Training Details}
\label{app:training_details}

\subsection{SFT Warm-Start}
\label{app:sft_details}

All RL methods (GRPO, F-GRPO) are initialized with a supervised fine-tuning (SFT) warm-start. For Qwen3.5-2B on QA tasks, RL training starts from the pretrained model directly.
Because our tasks require selecting items from a fixed candidate pool and producing structured \texttt{<SLATE>}/\texttt{<RANK>} output, the pretrained model has a high malformed-output rate without SFT (missing tags, hallucinated item names), leading to sparse rewards that slow RL convergence.
Across tasks, supervision is used only to construct SFT targets and compute rewards during training; inference receives only the prompt and candidate pool.

\subsubsection{Factorized SFT (Ours)}

The factorized SFT trains on gold completions with both \texttt{<SLATE>} and \texttt{<RANK>} tags.
Given a sample with gold item(s) $\mathcal{G}$ and distractor pool $\mathcal{N}$:
\begin{enumerate}
    \item Set slate size $n = 10$ (fixed, matching the output size cap \texttt{max\_slate\_items}).
    \item Construct slate: all gold items $\mathcal{G}$ plus $\min(n - |\mathcal{G}|, |\mathcal{N}|)$ randomly sampled distractors, then shuffle.
    \item Construct rank: gold items first (shuffled), then distractors (shuffled), truncated to \texttt{max\_rank\_items} (default 5).
\end{enumerate}

\begin{tcolorbox}[colback=gray!5,colframe=gray!75,title=Example Gold Completion (Recommendation --- Factorized)]
\small
\texttt{<SLATE>Dist.~F, Dist.~B, \textbf{Gold}, Dist.~A, Dist.~C, Dist.~H, Dist.~D, Dist.~G, Dist.~E, Dist.~I</SLATE>}\\
\texttt{<RANK>\textbf{Gold}, Dist.~F, Dist.~B, Dist.~A, Dist.~C</RANK>}
\end{tcolorbox}

\subsubsection{Baseline SFT}

The baseline SFT trains on gold completions with only the \texttt{<RANK>} tag:

\begin{tcolorbox}[colback=gray!5,colframe=gray!75,title=Example Gold Completion (Recommendation --- Baseline)]
\small
\texttt{<RANK>\textbf{Gold}, Distractor A, Distractor B, Distractor C, Distractor D</RANK>}
\end{tcolorbox}

Gold items are listed first, followed by randomly sampled distractors, truncated to 5 candidates.

\subsubsection{Decoupled SFT}
\label{app:decoupled_sft}

The decoupled SFT baseline trains \emph{separate specialist models} for selection and ranking, then chains them at inference via two sequential LLM calls.
This isolates the benefit of joint training within a shared backbone (as in F-GRPO) versus independently optimized specialists.

\paragraph{Selector training.}
The selector is trained on \texttt{<SLATE>}-only gold completions using the Step-1 prompts from the two-step zero-shot baseline (Section~\ref{app:zeroshot_prompts}).
Gold slates contain all gold items plus randomly sampled distractors (shuffled), matching the factorized SFT slate construction.

\paragraph{Ranker training.}
The ranker is trained on \texttt{<RANK>}-only gold completions using Step-2 prompts conditioned on gold slates.
Gold ranks list gold items first (shuffled), then distractors (shuffled), truncated to 5 candidates.
Critically, the ranker always sees \emph{gold slates} during training.

\paragraph{Inference.}
At inference, the selector generates a \texttt{<SLATE>}, which is parsed and fed as context to the ranker's Step-2 prompt.
The ranker then generates a \texttt{<RANK>}.

\paragraph{Three variants.}
\begin{itemize}
    \item \textbf{Selector only}: Trained selector + base model ranker (zero-shot).
    \item \textbf{Ranker only}: Base model selector (zero-shot) + trained ranker.
    \item \textbf{Full}: Both selector and ranker trained independently.
\end{itemize}
This baseline applies to recommendation and QA tasks. Both selector and ranker use \emph{best eval loss} checkpoint selection, matching the other SFT baselines.

\subsubsection{SFT Hyperparameters}

All SFT variants share the same training recipe, summarized in Table~\ref{tab:sft_hparams}.

\begin{table}[t]
\centering
\small
\begin{tabular}{l l}
\toprule
\multicolumn{2}{c}{\textbf{Training Setup}} \\
\midrule
Learning rate & $2 \times 10^{-5}$ (3\% warmup) \\
Epochs & 1 \\
Optimizer & AdamW \\
Precision & BF16 \\
Batch size & 2 $\times$ 8 accumulation (16) \\
Max sequence length & 2048 \\
Max completion length & 256 (Rec/QA) \\
Training samples (Recommendation, MuSiQue) & 10k (200 val) \\
Checkpoint selection & Best eval loss \\
\bottomrule
\end{tabular}
\caption{Shared SFT training hyperparameters for baseline, factorized, and decoupled models.}
\label{tab:sft_hparams}
\end{table}

The SFT prompt templates match the corresponding GRPO prompts exactly (factorized or baseline; see Section~\ref{app:prompts}). After SFT, the checkpoint is loaded as the starting policy for GRPO training with learning rate $5 \times 10^{-6}$ ($2 \times 10^{-6}$ for Qwen3.5-2B on QA tasks).

\subsection{GRPO Hyperparameters}
\label{app:hyperparams}

Table~\ref{tab:grpo_hparams} summarizes the GRPO training configuration.

\begin{table}[t]
\centering
\small
\begin{tabular}{l l}
\toprule
\multicolumn{2}{c}{\textbf{Training Setup}} \\
\midrule
Learning rate & $5 \times 10^{-6}$ (3\% warmup); $2 \times 10^{-6}$ for Qwen3.5-2B QA \\
Batch size & 2 $\times$ 8 accumulation (16) \\
Epochs & 1 \\
RL algorithm & Dr.~GRPO \\
Group size & 8 rollouts per prompt \\
Temperature & 1.0 \\
Top-$p$ & 0.9 \\
Max completion length & 256 (Rec/QA) \\
Clipping & $\epsilon_{\mathrm{clip}} = 0.2$ \\
KL weight & $\beta_{\mathrm{KL}} = 0.0$ \\
Ranking trade-off & $\lambda = 1.0$ \\
Training samples & 10k \\
Output size caps & rec/QA: \texttt{max\_slate\_items} = 10, \texttt{max\_rank\_items} = 5 \\
Oversize penalty & $-0.5$ per phase \\
Evaluation & Every 50 steps on 200 val prompts; final test on 2000 samples \\
\bottomrule
\end{tabular}
\caption{Shared GRPO hyperparameters for baseline and factorized models.}
\label{tab:grpo_hparams}
\end{table}

Following recent works~\citep{yu2025dapo,liu2025understanding,hu2025openreasonerzero}, we set $\beta_{\mathrm{KL}} = 0$.

Evaluation uses greedy decoding ($T{=}0$, no sampling) with $G_{\mathrm{eval}}{=}1$ (deterministic single-prediction evaluation) for all tasks.

\subsection{Ablation Results}
\label{app:ablation_results}

\paragraph{Hyperparameter sensitivity.}
Figure~\ref{fig:hyperparameter_ablation}a shows that $\lambda{=}1.0$ yields the best or near-best performance across all metrics; underweighting the ranking loss ($\lambda{=}0.5$) consistently degrades top-position quality, confirming that the ranker needs sufficient gradient signal.
Figure~\ref{fig:hyperparameter_ablation}b varies the maximum slate size $n \in \{5, 10, 15\}$. All metrics peak at $n{=}10$: $n{=}5$ limits candidate coverage, while $n{=}15$ dilutes the pool with low-relevance items, slightly degrading Precision@1 and NDCG@1. We use $\lambda{=}1.0$ and $n{=}10$ as defaults.

\begin{figure}[h]
  \centering
  \begin{subfigure}[b]{\linewidth}
    \centering
    \includegraphics[width=0.8\linewidth]{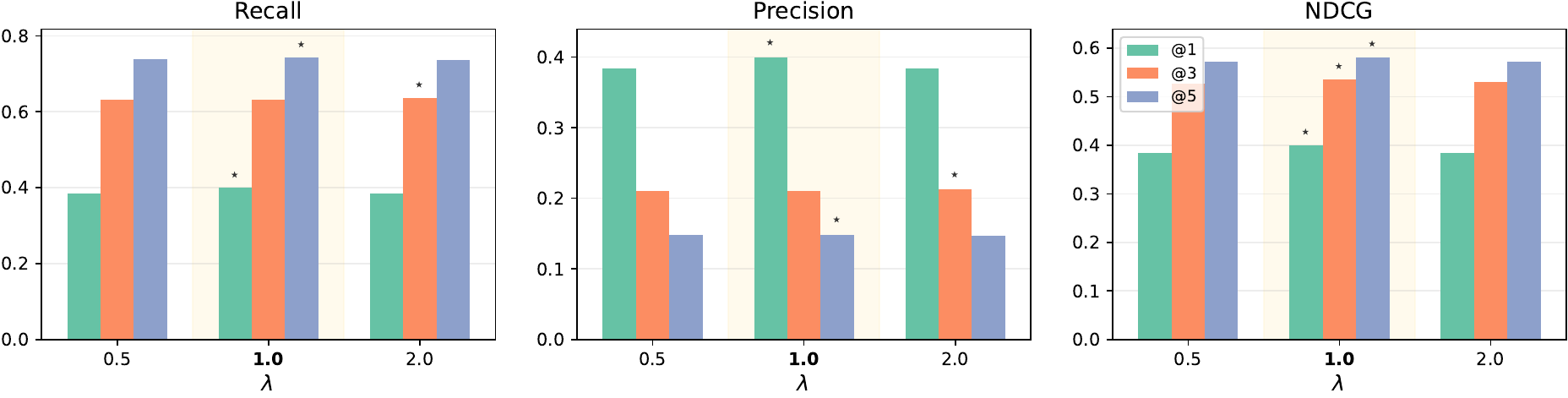}
    \caption{Effect of $\lambda$ on LastFM (Qwen3.5-2B). Stars mark the best setting per cutoff.}
    \label{fig:lambda_ablation}
  \end{subfigure}
  \vspace{0.3em}
  \begin{subfigure}[b]{\linewidth}
    \centering
    \includegraphics[width=0.7\linewidth]{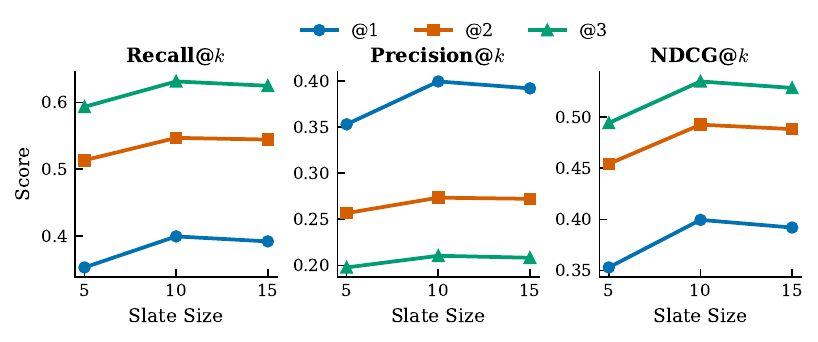}
    \caption{Effect of slate size on LastFM (Qwen3.5-2B).}
    \label{fig:slate_size}
  \end{subfigure}
  \caption{Hyperparameter sensitivity for F-GRPO. (a) $\lambda{=}1.0$ yields the best or near-best performance across all metrics; underweighting the ranking loss ($\lambda{=}0.5$) degrades top-position quality. (b) Performance peaks at slate size $n{=}10$; smaller slates limit coverage while larger slates dilute the candidate pool.}
  \label{fig:hyperparameter_ablation}
\end{figure}

\section{Additional Results}
\label{app:eval_metrics}
\label{app:results}

The main text reports Recall@$k$ and NDCG@$k$. For completeness, we include the remaining metrics here.
Table~\ref{tab:precision} reports Precision@$k$ for recommendation (LastFM and MovieLens).
\input{fig/results_rec_precision_appendix}

Table~\ref{tab:qa_precision_hit} reports Precision@$k$ and Hit@$k$ for multi-hop QA (MuSiQue and HotpotQA). The trends are consistent with those reported in the main text: F-GRPO achieves the best or second-best performance among LLM-based methods across all cutoffs.
\input{fig/results_qa_appendix}

\subsection{Optimization dynamics.}
\label{app:optimization_dynamics}
\Cref{fig:training_curves} compares F-GRPO and GRPO evaluation metrics over training on Qwen3-4B.
F-GRPO converges to higher values on most metrics on both datasets, with the clearest gains at higher cutoffs.
The advantage is most pronounced on Recall@5 and NDCG@5, where the factorized slate's broader coverage compounds with ranking quality: on MovieLens, F-GRPO reaches Recall@5${\approx}$0.70 while GRPO plateaus at${\approx}$0.60.
F-GRPO also converges faster, achieving strong performance within the first 1,000 steps on LastFM, while GRPO requires${\approx}$2,000 steps to reach comparable Recall@1.
This faster convergence is consistent with the credit-assignment analysis in Section~\ref{sec:method}: by eliminating the cross-phase gradient contamination identified in~\Cref{eq:joint_advantage}, each phase receives a cleaner learning signal, reducing the number of updates needed to reach high reward.
\begin{figure*}[t]
  \centering
  \includegraphics[width=\textwidth]{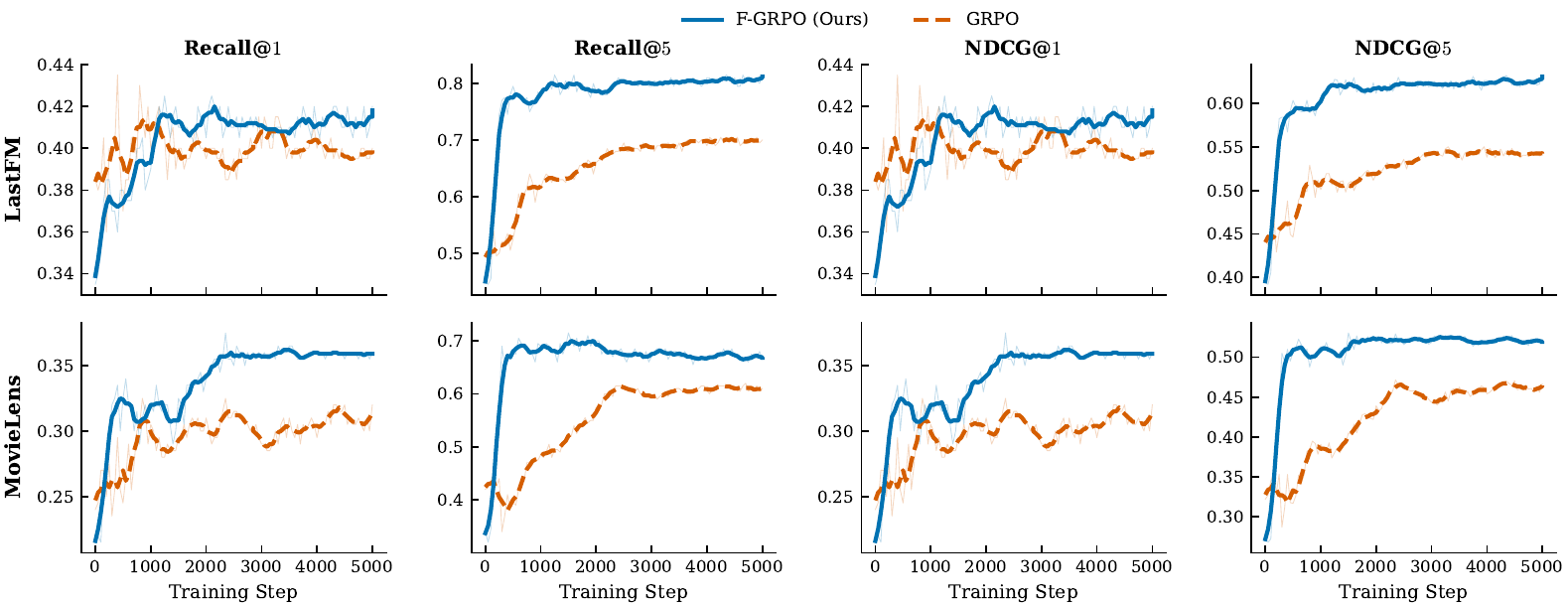}
  \caption{Evaluation metrics during training for F-GRPO and GRPO on Qwen3-4B. F-GRPO converges faster and reaches higher values on most metrics, with the clearest gains at higher cutoffs. Faint lines show raw values, bold lines are smoothed.}
  \label{fig:training_curves}
\end{figure*}

\subsection{Error attribution}
\label{app:error_attribution}

We decompose F-GRPO's test errors into two failure modes: \emph{slate miss} (the gold item never appears in the slate) and \emph{rank drop} (the gold item appears in the slate but the ranker fails to surface it).

On both datasets, errors are distributed across both phases rather than dominated by one, validating the factorized design. If errors were concentrated in a single phase, phase-specific advantages would offer little benefit over GRPO.

\begin{table}[t]
\centering
\small
\begin{tabular}{lccc}
\toprule
Dataset & Slate miss & Rank drop & Success \\
\midrule
LastFM & 8.9\% & 9.4\% & 81.8\% \\
MovieLens & 17.9\% & 15.4\% & 66.6\% \\
\bottomrule
\end{tabular}
\caption{Error attribution for F-GRPO (Qwen3-4B). Errors are distributed across both phases rather than dominated by one, confirming that each requires its own credit signal.}
\label{tab:error_attribution}
\end{table}

\section{Representative Examples}
\label{app:qa_examples}

We provide two representative HotpotQA examples from Qwen3-4B. For readability, we show passage titles rather than abstract IDs (P1--P20).

\begin{tcolorbox}[colback=gray!5,colframe=gray!75,title=HotpotQA Example 1 (Qwen3-4B)]
\small
\textbf{Question:} Shahnoza Nazirova traveled to Icheon and Ansan, South Korea to compete in what sport and event?\\[0.3em]
\textbf{Gold truth:} Shahnoza Nazirova; Volleyball at the 2014 Asian Games -- Women\\[0.3em]
\textbf{F-GRPO \texttt{<SLATE>}:} \texttt{Shahnoza Nazirova; Volleyball at the 2014 Asian Games -- Women; Icheon}\\
\textbf{F-GRPO \texttt{<RANK>}:} \texttt{Shahnoza Nazirova; Volleyball at the 2014 Asian Games -- Women}\\
\textbf{Baseline \texttt{<RANK>}:} \texttt{Icheon; Shahnoza Nazirova; Ansan}\\[0.3em]
\end{tcolorbox}

\begin{tcolorbox}[colback=gray!5,colframe=gray!75,title=HotpotQA Example 2 (Qwen3-4B)]
\small
\textbf{Question:} What Russian city are both Mariinsky Ballet and Arthur Saint-L\'eon located?\\[0.3em]
\textbf{Gold truth:} Mariinsky Ballet; Arthur Saint-L\'eon\\[0.3em]
\textbf{F-GRPO \texttt{<SLATE>}:} \texttt{Mariinsky Ballet; Mariinsky Theatre; Arthur Saint-L\'eon}\\
\textbf{F-GRPO \texttt{<RANK>}:} \texttt{Mariinsky Ballet; Arthur Saint-L\'eon}\\
\textbf{Baseline \texttt{<RANK>}:} \texttt{Mariinsky Ballet; Larissa Lezhnina; Mariinsky Theatre}\\[0.3em]
\end{tcolorbox}

\section{Prompt Templates}
\label{app:prompts}

We reproduce the full prompt templates used for all tasks. The factorized prompt instructs the model to output both a slate and a ranking in a single pass, wrapped in \texttt{<SLATE>...</SLATE>} and \texttt{<RANK>...</RANK>} tags; the baseline prompt uses only \texttt{<RANK>} tags. Both include explicit size guidance (``up to 10'' for the slate, ``top 5'' for the rank), matching the output size caps enforced during training (Appendix~\ref{app:hyperparams}).

\subsection{Recommendation (MovieLens \& LastFM)}

\subsubsection*{Factorized Mode}

\begin{tcolorbox}[colback=gray!5,colframe=gray!75,title=Recommendation Factorized Prompt Template]
\small
\textbf{System:} You are a recommendation system. Given a user's consumption history, predict which items they are most likely to select next from the provided candidates.\\[0.5em]
\textbf{User:}\\
\texttt{User history:}\\
\texttt{\{history\}}\\[0.3em]
\texttt{Candidate Items:}\\
\texttt{\{candidates\}}\\[0.3em]
Based on the user's history, predict which items they are likely to select next.\\
You MUST follow this exact format. Output ONLY the tagged blocks.\\
Rules:
\begin{itemize}
\item SLATE: Select up to \{max\_slate\} items the user would most likely choose next.
\item RANK: Rank your top \{max\_rank\} items from SLATE, from most to least likely to be their next choice.
\item Use ONLY item names from the candidates list (no paraphrases).
\end{itemize}
Output format:\\
\texttt{<SLATE>Item A, Item B, Item C, Item D</SLATE>}\\
\texttt{<RANK>Item B, Item A, Item D, Item C</RANK>}
\end{tcolorbox}

\subsubsection*{Baseline Mode}

\begin{tcolorbox}[colback=gray!5,colframe=gray!75,title=Recommendation Baseline Prompt Template]
\small
\textbf{System:} You are a recommendation system. Given a user's consumption history, predict which items they are most likely to select next from the provided candidates.\\[0.5em]
\textbf{User:}\\
\texttt{User history:}\\
\texttt{\{history\}}\\[0.3em]
\texttt{Candidate Items:}\\
\texttt{\{candidates\}}\\[0.3em]
Based on the user's history, predict which items they are likely to select next.\\
You MUST follow this exact format. Output ONLY the tagged block.\\
Rules:
\begin{itemize}
\item RANK: Select and rank the top \{max\_rank\} items from most to least likely to be their next choice.
\item Use ONLY item names from the candidates list (no paraphrases).
\end{itemize}
Output format:\\
\texttt{<RANK>Item B, Item A, Item D, Item C</RANK>}
\end{tcolorbox}

\paragraph{Note.}
For LastFM, ``items'' refer to artist names instead of movie titles. The prompt structure remains identical.

\subsection{Question Answering (HotpotQA \& MuSiQue)}

\subsubsection*{Factorized Mode}

\begin{tcolorbox}[colback=gray!5,colframe=gray!75,title=QA Factorized Prompt Template]
\small
\textbf{System:} You are an evidence selection system. Given a question and candidate passages (each with an ID and content), identify which passage IDs contain supporting evidence.\\[0.5em]
\textbf{User:}\\
\texttt{Query: \{question\}}\\[0.3em]
\texttt{Candidate Passages:}\\
\texttt{\{candidates\}}\\[0.3em]
You MUST follow this exact format. Output ONLY the tagged blocks.\\
Rules:
\begin{itemize}
\item SLATE: Select up to \{max\_slate\} passage IDs that contain relevant evidence.
\item RANK: Rank your top \{max\_rank\} passage IDs from SLATE, from most to least relevant.
\item Use ONLY the passage IDs listed above (do not invent IDs).
\end{itemize}
Output format:\\
\texttt{<SLATE>P1, P2, P3</SLATE>}\\
\texttt{<RANK>P2, P1, P3</RANK>}
\end{tcolorbox}

\subsubsection*{Baseline Mode}

\begin{tcolorbox}[colback=gray!5,colframe=gray!75,title=QA Baseline Prompt Template]
\small
\textbf{System:} You are an evidence selection system. Given a question and candidate passages (each with an ID and content), identify which passage IDs contain supporting evidence.\\[0.5em]
\textbf{User:}\\
\texttt{Query: \{question\}}\\[0.3em]
\texttt{Candidate Passages:}\\
\texttt{\{candidates\}}\\[0.3em]
You MUST follow this exact format. Output ONLY the tagged block.\\
Rules:
\begin{itemize}
\item RANK: Select and rank the top \{max\_rank\} passage IDs from most to least relevant.
\item Use ONLY the passage IDs listed above (do not invent IDs).
\end{itemize}
Output format:\\
\texttt{<RANK>P2, P1, P3</RANK>}
\end{tcolorbox}

\paragraph{Note.}
Candidate passages are formatted as \texttt{**P1**} followed by title and text content. For HotpotQA, passage IDs correspond to Wikipedia abstract titles; for MuSiQue, passages are numbered P1--P20.

\section{Two-Step Zero-Shot Prompt Templates}
\label{app:zeroshot_prompts}

The two-step zero-shot baseline decomposes the list-to-rank task into two sequential LLM calls without any training: Step~1 generates a \texttt{<SLATE>} (candidate selection), and Step~2 feeds the generated slate back as context to produce a \texttt{<RANK>} (ranking).
Both steps use greedy decoding ($T{=}0$) with a single rollout per prompt.
System prompts are shared with the factorized templates above; we show only the user messages.

\subsection{Recommendation}

\begin{tcolorbox}[colback=gray!5,colframe=gray!75,title=Recommendation --- Step 1 (Slate Generation)]
\small
\texttt{User history:}\\
\texttt{\{history\}}\\[0.3em]
\texttt{Candidate Items:}\\
\texttt{\{candidates\}}\\[0.3em]
Based on the user's history, select items they are most likely to choose next.\\
You MUST follow this exact format. Output ONLY the tagged block.\\
Rules:
\begin{itemize}
\item SLATE: Select up to \{max\_slate\} items the user would most likely choose next.
\item Use ONLY item names from the candidates list (no paraphrases).
\end{itemize}
Output format:\\
\texttt{<SLATE>Item A, Item B, Item C, Item D</SLATE>}
\end{tcolorbox}

\begin{tcolorbox}[colback=gray!5,colframe=gray!75,title=Recommendation --- Step 2 (Ranking)]
\small
\texttt{User history:}\\
\texttt{\{history\}}\\[0.3em]
\texttt{Candidate Items:}\\
\texttt{\{candidates\}}\\[0.3em]
Based on the user's history, you previously selected these items as likely next choices: \{slate\_items\}\\[0.3em]
You MUST follow this exact format. Output ONLY the tagged block.\\
Rules:
\begin{itemize}
\item RANK: Rank the top \{max\_rank\} items from most to least likely.
\item Use ONLY item names from your selection (no paraphrases).
\end{itemize}
Output format:\\
\texttt{<RANK>Item B, Item A, Item D, Item C</RANK>}
\end{tcolorbox}

\subsection{Question Answering}

\begin{tcolorbox}[colback=gray!5,colframe=gray!75,title=QA --- Step 1 (Slate Generation)]
\small
\texttt{Query: \{question\}}\\[0.3em]
\texttt{Candidate Passages:}\\
\texttt{\{candidates\}}\\[0.3em]
You MUST follow this exact format. Output ONLY the tagged block.\\
Rules:
\begin{itemize}
\item SLATE: Select up to \{max\_slate\} passage IDs that contain relevant evidence.
\item Use ONLY the passage IDs listed above (do not invent IDs).
\end{itemize}
Output format:\\
\texttt{<SLATE>P1, P2, P3</SLATE>}
\end{tcolorbox}

\begin{tcolorbox}[colback=gray!5,colframe=gray!75,title=QA --- Step 2 (Ranking)]
\small
\texttt{Query: \{question\}}\\[0.3em]
\texttt{Candidate Passages:}\\
\texttt{\{candidates\}}\\[0.3em]
Based on the query, you previously selected these passage IDs as containing relevant evidence: \{slate\_items\}\\[0.3em]
You MUST follow this exact format. Output ONLY the tagged block.\\
Rules:
\begin{itemize}
\item RANK: Rank the top \{max\_rank\} passage IDs from most to least relevant.
\item Use ONLY passage IDs from your selection (do not invent IDs).
\end{itemize}
Output format:\\
\texttt{<RANK>P2, P1, P3</RANK>}
\end{tcolorbox}

\section{Extended Related Work}
\label{app:related_extended}

This section expands on the discussion in Section~\ref{sec:related}, providing additional context for each line of work.

\subsection{LLM-based Retrieval and Reranking}

Recent work on LLM-based retrieval broadly splits into generative retrieval and reranking over a fixed candidate pool. In generative retrieval, language models directly emit document identifiers rather than independently scoring a corpus: Autoregressive Search Engines~\citep{bevilacqua2022autoregressive} generates corpus-constrained n-gram substrings as identifiers using an FM-Index, Neural Corpus Indexer~\citep{wang2022neural} learns hierarchical semantic document identifiers via clustering, and ListGR~\citep{10.1145/3653712} extends the formulation with sequential, position-aware listwise training so the generation objective better matches ranked evaluation metrics.

When a candidate set is already available, LLMs have been applied as increasingly strong rerankers. MonoT5~\citep{nogueira-etal-2020-document} and the Expando-Mono-Duo framework~\citep{pradeep2021expandomonoduodesignpatterntext} showed that seq2seq models serve as effective pointwise and pairwise rerankers within classical multi-stage pipelines. Subsequent work moved toward listwise formulations. RankGPT~\citep{sun-etal-2023-chatgpt} showed that proprietary LLMs can perform zero-shot listwise reranking competitive with supervised methods, and subsequent open-weight efforts distill this capability into smaller models: RankZephyr~\citep{pradeep2023rankzephyreffectiverobustzeroshot} into a 7B model and LiT5~\citep{tamber2023scalingdownlittingup} into 220M--3B parameters via a Fusion-in-Decoder architecture.

More recent methods optimize ranking objectives directly with RL or preference learning. Neural PG-RANK~\citep{gao2024policygradienttraininglanguagemodels} casts retrieval as a Plackett-Luce ranking policy trained with REINFORCE. RaCT~\citep{10.1145/3767695.3769487} combines chain-of-thought supervised fine-tuning with ranking preference optimization. Rank-R1~\citep{zhuang2025rankr1enhancingreasoningllmbased} applies GRPO to a setwise reranker that selects the single most relevant document per call, eliciting reasoning without explicit supervision. Rank-GRPO~\citep{zhu2026rankgrpotrainingllmbasedconversational} adapts GRPO for conversational recommendation by elevating credit assignment from the token level to individual rank positions within a single list. IRPO~\citep{wu2025incontext} extends DPO to listwise ranking with a positional preference model inspired by Plackett-Luce.

Closest to our setting are methods that treat candidate generation and final ordering as coupled decisions rather than isolated stages. Two-stage counterfactual learning to rank~\citep{10.1145/3731120.3744583} jointly optimizes a candidate generator and ranker under logged feedback via alternating updates, while GeMS~\citep{10.1145/3539597.3570412} and HiGR~\citep{pang2026higrefficientgenerativeslate} optimize whole recommendation slates rather than independent items, using latent-space RL and hierarchical planning with listwise preference alignment, respectively. Our work follows this direction but differs in realizing both phases within a single autoregressive LLM rollout: the same model first constructs a slate and then ranks it, with separate learning signals for slate quality and within-slate ordering rather than a single shared sequence-level objective.

\subsection{Reinforcement Learning for LLMs}

Reinforcement learning for LLM post-training largely builds on policy-gradient methods. PPO~\citep{schulman2017proximalpolicyoptimizationalgorithms}, with its clipped surrogate objective, was scaled to LLM alignment in RLHF~\citep{ouyang2022training}, where a learned reward model guides policy updates. Recent reasoning-focused work has increasingly favored critic-free variants that avoid learning a separate value model.~\citet{shao2024deepseekmathpushinglimitsmathematical} introduced GRPO, which replaces the critic with group-relative reward normalization over multiple samples from the same prompt, making large-scale RL more practical for verifiable reasoning tasks;~\citet{guo2025deepseek} subsequently scaled this approach to frontier reasoning.

Subsequent work mainly refines how credit is assigned within this critic-free framework. Dr.\ GRPO~\citep{liu2025understanding} shows that the standard-deviation normalization in vanilla GRPO introduces response-length and question-difficulty biases, and proposes an unbiased variant that removes the denominator, improving stability and token efficiency. GVPO~\citep{zhang2025gvpogroupvariancepolicy} revisits the weighting rule more fundamentally, deriving group weights from the KL-constrained reward-maximization problem so that the intractable partition function cancels, yielding provably optimal updates. DAPO~\citep{yu2025dapo} shows that large-scale reasoning RL is highly sensitive to training recipe choices, improving long-chain-of-thought training with decoupled clipping, dynamic sampling, token-level policy loss, and overlong reward shaping. MiniMax-M1~\citep{minimax2025minimaxm1scalingtesttimecompute} introduces CISPO, which clips importance-sampling weights rather than zeroing out per-token gradient contributions, preserving entropy and stabilizing training for long generations.

Alongside these online RL methods, work on offline and alternative objectives has broadened the design space. DPO~\citep{10.5555/3666122.3668460} reparameterizes the reward under a Bradley-Terry preference model so that the partition function cancels, enabling direct optimization from pairwise preferences without an explicit reward model or RL loop. MaxRL~\citep{tajwar2026maximumlikelihoodreinforcementlearning} argues that expected-reward RL optimizes only a first-order approximation to the underlying maximum-likelihood objective over correct rollouts, and proposes a compute-indexed family of objectives that interpolate toward exact likelihood as sampling budget grows. Our method is complementary: whereas prior LLM RL methods improve the objective for a single monolithic completion, we study a structured rollout containing two coupled decisions (candidate generation and ranking) and therefore require phase-specific credit assignment rather than one shared sequence-level advantage.

%% file: fig/algo.tex
\begin{algorithm}[t]
\caption{Factorized List-and-Rank Training with Phase-Specific Group-Relative Advantages}
\label{alg:training}
\scriptsize
\begin{algorithmic}[1]
\Require Training dataset $\mathcal{D}_{\mathrm{train}}$, group size $G$, clip $\epsilon_{\mathrm{clip}}$, KL coef.\ $\beta_{\mathrm{KL}}$, trade-off $\lambda$, lr $\alpha$, temperature $T$, steps $M$, format penalty $p_{\mathrm{fmt}}$
\State Initialize policy parameters $\theta$ from SFT checkpoint; optionally fix reference $\pi_{\mathrm{ref}} \leftarrow \pi_\theta$ if $\beta_{\mathrm{KL}} > 0$
\For{$\mathrm{step}=1,\dots,M$}
  \State Sample minibatch $\mathcal{B}\subset\mathcal{D}_{\mathrm{train}}$
  \State $\theta_{\mathrm{old}} \leftarrow \theta$ \Comment{Snapshot current policy}
  \ForAll{$x\in\mathcal{B}$} \Comment{Rollout collection}
    \For{$i=1$ to $G$}
      \State Sample rollout $o^{(i)} \sim \pi_{\theta_{\mathrm{old}}}(\cdot\mid x)$ with temperature $T$
      \State Parse $o^{(i)}$ into slate content $c^{(i)}_{\tau}$ and rank content $c^{(i)}_{\sigma}$ via delimiter tags
      \If{both \texttt{<SLATE>} and \texttt{<RANK>} present} \Comment{Case 1}
        \State Compute $R^{(i)}_{\mathrm{slate}}$, $R^{(i)}_{\mathrm{rank}}$ from task-specific rewards
      \ElsIf{only \texttt{<SLATE>} present} \Comment{Case 2}
        \State Compute $R^{(i)}_{\mathrm{slate}}$; set $R^{(i)}_{\mathrm{rank}} \leftarrow p_{\mathrm{fmt}}$
      \Else \Comment{Case 3: malformed}
        \State $R^{(i)}_{\mathrm{slate}} \leftarrow p_{\mathrm{fmt}}$, $R^{(i)}_{\mathrm{rank}} \leftarrow p_{\mathrm{fmt}}$
      \EndIf
    \EndFor
    \State Compute group means $\bar{R}_{\mathrm{slate}}$ and $\bar{R}_{\mathrm{rank}}$ over $G$ rollouts
    \For{$i=1$ to $G$} \Comment{Factorized advantage estimation}
      \State $\hat{A}^{(i)}_{\mathrm{slate}} \gets R^{(i)}_{\mathrm{slate}}-\bar{R}_{\mathrm{slate}}$
      \State $\hat{A}^{(i)}_{\mathrm{rank}} \gets R^{(i)}_{\mathrm{rank}}-\bar{R}_{\mathrm{rank}}$
    \EndFor
  \EndFor
  \State \textbf{Phase 1 (Slate):} Compute per-token log-probs $\log \pi_\theta(c^{(i)}_{\tau,t} \mid x, c^{(i)}_{\tau,<t})$ for all collected slate segments
  \State \hspace{\algorithmicindent} Apply delimiter masking: loss on content tokens only (exclude \texttt{<SLATE>}, \texttt{</SLATE>})
  \State \hspace{\algorithmicindent} Compute $\mathcal{L}_{\mathrm{slate}}$ via per-token clipped loss (Eq.~\eqref{eq:slate_loss}) with $\hat{A}^{(i)}_{\mathrm{slate}}$
  \State \textbf{Phase 2 (Rank):} Compute per-token log-probs $\log \pi_\theta(c^{(i)}_{\sigma,t} \mid x, \tau^{(i)}, c^{(i)}_{\sigma,<t})$ for Case~1 rank segments and Case~2 post-slate continuations
  \State \hspace{\algorithmicindent} Apply delimiter masking for Case~1; use the raw continuation without delimiter masking for Case~2
  \State \hspace{\algorithmicindent} Compute $\mathcal{L}_{\mathrm{rank}}$ via per-token clipped loss (Eq.~\eqref{eq:rank_loss}) with $\hat{A}^{(i)}_{\mathrm{rank}}$
  \State $\mathcal{L}(\theta) \gets \mathcal{L}_{\mathrm{slate}} + \lambda\,\mathcal{L}_{\mathrm{rank}} + \beta_{\mathrm{KL}}\,D_{\mathrm{KL}}(\pi_\theta \| \pi_{\mathrm{ref}})$
  \State $\theta \leftarrow \theta - \alpha \nabla_\theta \mathcal{L}(\theta)$ \Comment{Single gradient step per batch}
\EndFor
\State \Return Trained policy $\pi_\theta$
\end{algorithmic}
\end{algorithm}

%% file: fig/results_rec_precision_appendix.tex
\begin{table}[t]
  \centering
  \resizebox{0.5\textwidth}{!}{%
  \begin{tabular}{@{}lccc@{\hspace{0.8em}}!{\vrule width 0.3pt}@{\hspace{0.8em}}ccc@{}}
  \toprule
  & \multicolumn{3}{c}{LastFM} & \multicolumn{3}{c}{MovieLens} \\
  \cmidrule(lr){2-4} \cmidrule(lr){5-7}
  Method & @1 & @3 & @5 & @1 & @3 & @5 \\
  \midrule
  Random & 5.0 & 5.0 & 5.0 & 4.9 & 5.1 & 5.0 \\
  Popularity & 24.7 & 15.3 & 11.3 & 14.4 & 12.1 & 11.0 \\
  BM25 & 6.2 & 5.9 & 5.5 & 8.7 & 6.1 & 5.5 \\
  GRU4Rec & 50.8 & 23.0 & 15.1 & 21.8 & 15.6 & 12.0 \\
  UniSRec (zero-shot) & 11.0 & 8.5 & 7.4 & 6.3 & 5.5 & 5.1 \\
  UniSRec (fine-tuned) & 23.9 & 14.1 & 10.9 & 13.6 & 9.9 & 8.7 \\
  \addlinespace[3pt]
  \midrule
  \multicolumn{7}{@{}l}{\textbf{Qwen3-4B}} \\
  Zero-shot (two-step) & 18.8 & 14.9 & 11.7 & 12.1 & 9.7 & 8.0 \\
  Zero-shot (single-step) & 20.6 & 14.6 & 11.2 & 11.9 & 8.9 & 7.8 \\
  SFT & 40.1 & 15.7 & 10.8 & 21.5 & 9.8 & 7.3 \\
  \cdashline{1-7}\noalign{\vskip 2pt}
  Decoupled SFT (sel.\ only) & 14.1 & 11.1 & 9.2 & 10.1 & 8.4 & 7.2 \\
  Decoupled SFT (rank.\ only) & 28.8 & 14.3 & 10.4 & 16.5 & 9.3 & 7.4 \\
  Decoupled SFT (full) & 33.4 & 13.5 & 9.2 & 19.7 & 9.3 & 7.1 \\
  \cdashline{1-7}\noalign{\vskip 2pt}
  GRPO & \underline{44.7} & \underline{18.4} & \underline{14.8} & \underline{23.9} & \underline{14.8} & \underline{11.5} \\
  \textbf{F-GRPO (ours)} & \textbf{46.8} & \textbf{24.1} & \textbf{16.3} & \textbf{26.9} & \textbf{17.9} & \textbf{13.3} \\
  \midrule
  \addlinespace[3pt]
  \multicolumn{7}{@{}l}{\textbf{Qwen3.5-2B}} \\
  Zero-shot (two-step) & 0.6 & 0.5 & 0.5 & 1.2 & 0.6 & 0.4 \\
  Zero-shot (single-step) & 3.1 & 2.2 & 1.6 & 3.3 & 1.5 & 1.1 \\
  SFT & 37.9 & 15.1 & 10.5 & \textbf{27.7} & 11.6 & 8.2 \\
  \cdashline{1-7}\noalign{\vskip 2pt}
  Decoupled SFT (sel.\ only) & 0.5 & 0.4 & 0.4 & 0.7 & 0.6 & 0.5 \\
  Decoupled SFT (rank.\ only) & 23.1 & 10.5 & 7.7 & 12.3 & 6.9 & 5.4 \\
  Decoupled SFT (full) & 32.4 & 13.4 & 9.2 & 23.6 & 10.5 & 8.1 \\
  \cdashline{1-7}\noalign{\vskip 2pt}
  GRPO & \underline{38.7} & \underline{20.1} & \underline{14.3} & \underline{24.4} & \textbf{16.9} & \underline{11.3} \\
  \textbf{F-GRPO (ours)} & \textbf{40.0} & \textbf{21.1} & \textbf{14.9} & 21.5 & \underline{15.3} & \textbf{12.2} \\
  \bottomrule
  \end{tabular}%
  }
  \caption{Precision@$k$ on LastFM and MovieLens for Qwen3-4B and Qwen3.5-2B. All values are percentages. \textbf{Bold}: best among LLM methods; \underline{Underlined}: second best.}
  \label{tab:precision}
\end{table}

%% file: fig/results_qa_appendix.tex
\begin{table}[t]
  \centering
  \resizebox{\textwidth}{!}{%
  \begin{tabular}{@{}lcccccc@{\hspace{0.8em}}!{\vrule width 0.3pt}@{\hspace{0.8em}}cccccc@{}}
  \toprule
  & \multicolumn{6}{c}{MuSiQue} & \multicolumn{6}{c}{HotpotQA} \\
  \cmidrule(lr){2-7} \cmidrule(lr){8-13}
  & \multicolumn{3}{c}{Precision@$k$} & \multicolumn{3}{c}{Hit@$k$} & \multicolumn{3}{c}{Precision@$k$} & \multicolumn{3}{c}{Hit@$k$} \\
  \cmidrule(lr){2-4} \cmidrule(lr){5-7} \cmidrule(lr){8-10} \cmidrule(lr){11-13}
  Method & @1 & @3 & @5 & @1 & @3 & @5 & @1 & @3 & @5 & @1 & @3 & @5 \\
  \midrule
  Random & 13.3 & 13.3 & 13.2 & 13.3 & 36.0 & 53.8 & 9.9 & 10.0 & 10.0 & 9.9 & 28.5 & 44.7 \\
  BM25 & 57.0 & 36.5 & 27.7 & 57.0 & 81.5 & 89.6 & 83.6 & 48.1 & 33.1 & 83.6 & 96.6 & 98.6 \\
  RankZephyr & 82.3 & 51.2 & 36.6 & 82.3 & 95.0 & 97.6 & 94.8 & 56.6 & 36.1 & 94.8 & 98.9 & 99.6 \\
  MonoT5 & 86.1 & 54.8 & 38.6 & 86.1 & 97.9 & 99.4 & 92.5 & 57.5 & 37.1 & 92.5 & 99.0 & 99.8 \\
  DuoT5 & 84.8 & 53.2 & 37.9 & 84.8 & 96.7 & 99.1 & 86.4 & 53.2 & 34.5 & 86.4 & 92.6 & 94.9 \\
  LiT5 & 82.7 & 45.8 & 32.6 & 82.7 & 94.3 & 96.8 & 92.5 & 52.2 & 34.1 & 92.5 & 98.3 & 99.2 \\
  \addlinespace[3pt]
  \midrule
  \multicolumn{13}{@{}l}{\textbf{Qwen3-4B}} \\
  Zero-shot (two-step)         & 10.2 & 7.3 & 4.6 & 10.2 & 11.8 & 11.9 & 6.0 & 3.9 & 2.4 & 6.0 & 6.3 & 6.4 \\
  Zero-shot (single-step)      & 58.4 & 41.0 & 29.2 & 58.4 & 70.4 & 72.6 & 61.4 & 40.6 & 26.1 & 61.4 & 68.2 & 69.6 \\
  SFT                          & 69.5 & 46.3 & 31.4 & 69.5 & 91.3 & 94.5 & 76.6 & 45.8 & 29.2 & 76.6 & 90.1 & 92.5 \\
  \cdashline{1-13}\noalign{\vskip 2pt}
  Decoupled SFT (sel.\ only)   & 21.7 & 13.9 & 9.6 & 21.7 & 26.1 & 26.9 & 19.6 & 12.8 & 8.2 & 19.6 & 22.0 & 22.4 \\
  Decoupled SFT (rank.\ only)  & 79.8 & 53.6 & \underline{35.3} & 79.8 & 95.1 & \underline{96.9} & 87.6 & 56.2 & 35.7 & 87.6 & \underline{97.6} & 99.0 \\
  Decoupled SFT (full)         & 60.6 & 33.5 & 23.0 & 60.6 & 75.4 & 80.4 & 65.1 & 32.6 & 21.1 & 65.1 & 76.4 & 79.4 \\
  \cdashline{1-13}\noalign{\vskip 2pt}
  GRPO                         & \underline{90.8} & \underline{54.0} & 32.4 & \underline{90.8} & \underline{96.9} & \underline{96.9} & \textbf{96.1} & \textbf{62.0} & \textbf{37.2} & \textbf{96.1} & \textbf{99.4} & \textbf{99.5} \\
  \textbf{F-GRPO (ours)}       & \textbf{91.6} & \textbf{60.8} & \textbf{36.8} & \textbf{91.6} & \textbf{97.8} & \textbf{97.8} & \underline{96.0} & \underline{61.7} & \underline{37.0} & \underline{96.0} & \textbf{99.4} & \underline{99.4} \\
  \midrule
  \addlinespace[3pt]
  \multicolumn{13}{@{}l}{\textbf{Qwen3.5-2B}} \\
  Zero-shot (two-step)         & 5.4 & 3.4 & 2.4 & 5.4 & 7.0 & 7.5 & 2.0 & 1.3 & 0.8 & 2.0 & 2.4 & 2.5 \\
  Zero-shot (single-step)      & 13.2 & 9.2 & 6.5 & 13.2 & 17.8 & 18.9 & 8.0 & 5.0 & 3.4 & 8.0 & 9.6 & 9.9 \\
  SFT                          & 34.4 & 25.1 & 19.8 & 34.4 & 61.8 & 72.9 & 10.4 & 10.0 & 10.1 & 10.4 & 28.6 & 45.6 \\
  \cdashline{1-13}\noalign{\vskip 2pt}
  Decoupled SFT (sel.\ only)   & 0.7 & 0.4 & 0.4 & 0.7 & 1.0 & 1.3 & 0.5 & 0.4 & 0.3 & 0.5 & 0.9 & 0.9 \\
  Decoupled SFT (rank.\ only)  & 43.7 & 31.4 & 25.4 & 43.7 & 71.5 & 84.2 & 42.8 & 34.5 & 27.4 & 42.8 & \underline{75.4} & 87.6 \\
  Decoupled SFT (full)         & 13.8 & 13.1 & 13.0 & 13.8 & 35.5 & 52.2 & 11.9 & 11.1 & 10.3 & 11.9 & 31.3 & 45.4 \\
  \cdashline{1-13}\noalign{\vskip 2pt}
  GRPO                         & \textbf{90.0} & \underline{53.1} & \underline{32.9} & \textbf{90.0} & \underline{97.3} & \textbf{98.3} & \underline{96.1} & \underline{57.6} & \underline{34.6} & \underline{96.1} & \textbf{99.1} & \textbf{99.2} \\
  \textbf{F-GRPO (ours)}       & \underline{89.9} & \textbf{62.1} & \textbf{37.7} & \underline{89.9} & \textbf{97.8} & \underline{97.8} & \textbf{96.3} & \textbf{61.0} & \textbf{36.6} & \textbf{96.3} & \textbf{99.1} & \underline{99.1} \\
  \bottomrule
  \end{tabular}%
  }
  \caption{Precision@$k$ and Hit@$k$ on MuSiQue and HotpotQA for Qwen3-4B and Qwen3.5-2B. All values are percentages. \textbf{Bold}: best among LLM methods; \underline{Underlined}: second best.}
  \label{tab:qa_precision_hit}
\end{table}